\documentclass[oneside]{amsart}
\usepackage[T1]{fontenc}
\usepackage[latin9]{inputenc}
\usepackage{float}
\usepackage{amsmath}
\usepackage{amsthm}
\usepackage{amssymb}
\usepackage{array}
\usepackage{graphicx}
\usepackage{booktabs}
\usepackage{multirow}
\usepackage{subcaption}
\usepackage{hyperref}
\usepackage[square,numbers]{natbib}
\newcommand{\reals}{\mathbb{R}}
\newcommand{\Cov}{\mathrm{Cov}}
\newcommand{\sinc}{\mathrm{sinc}}

\newcommand{\comment}[1]{}
\newcommand{\realp}{\textrm{Re}}

\makeatletter

\floatstyle{ruled}
\newfloat{algorithm}{tbp}{loa}
\providecommand{\algorithmname}{Algorithm}
\floatname{algorithm}{\protect\algorithmname}

\numberwithin{equation}{section}
\numberwithin{figure}{section}
\theoremstyle{plain}

\newtheorem{thm}{Theorem}
\newtheorem{lem}[thm]{Lemma}
\newtheorem{prop}[thm]{Proposition}
\newtheorem{cors}[thm]{Corollary}
\newtheorem*{rem*}{Remark}

\makeatother

\begin{document}

\title{Super-resolution estimation of cyclic arrival rates}
\author{Ningyuan Chen$^{\dagger}$, Donald K.K. Lee$^{\ast}$, Sahand N. Negahban$^{\ast}$}
\thanks{Correspondence: Donald Lee (donald.lee@emory.edu)}

\maketitle
\begin{center}
HKUST$^{\dagger}$ and Yale University$^{\ast}$
\par\end{center}

\vspace*{0.6cm}
\begin{center}
	Preprint of \href{https://projecteuclid.org/euclid.aos/1550026856}{\textit{Annals of Statistics} 47:3:1754-1775 (2019)}
\end{center}
\vspace*{0.6cm}

\begin{abstract}
Exploiting the fact that most arrival processes exhibit cyclic behaviour, we propose a simple procedure for estimating the intensity of a nonhomogeneous Poisson process. The estimator is the super-resolution analogue to Shao and Lii \citep{shao2010,shao2011}, which is a sum of $p$ sinusoids where $p$ and the amplitude and phase of each wave are not known and need to be estimated. This results in an interpretable yet flexible specification that is suitable for use in modelling as well as in high resolution simulations.

Our estimation procedure sits in between classic periodogram methods and atomic/total variation norm thresholding. Through a novel use of window functions in the point process domain, our approach attains super-resolution without semidefinite programming. Under suitable conditions, finite sample guarantees can be derived for our procedure. These resolve some open questions and expand existing results in spectral estimation literature.
\end{abstract}



\section{Introduction\label{sec:intro}}

Real world arrival patterns typically exhibit cyclic (but not necessarily periodic) behaviour. Motivated by the need for tractable yet flexible functional forms for the arrival rate in queuing literature (\citet{CLS2018}), we consider the following problem: Suppose we observe the jump times $\{t_j\}_j$ of a nonhomogeneous
Poisson process (NHPP) $\{N(t):t\ge 0\}$ in $[0,T]$. Here, $N(t)$ denotes the number of
arrivals in $(0,t]$, and the intensity $\lambda(t)$ and the cumulative
rate function $\Lambda(t)$ are defined as
\[
\mathbb{E}N(t)=\int_{0}^{t}\lambda(u)du=\Lambda(t).
\]
Our goal is to use the observed data to estimate arrival rates of the form
\begin{equation}
\lambda(t)=c_{0}^{\lambda}+\sum_{j=1}^{p/2} d_{j}^{\lambda}\cos(f_{j}^{\lambda}t+\phi_{j}^{\lambda})=c_{0}^{\lambda}+\sum_{k=1}^{p}c_{k}^{\lambda}e^{2\pi i\nu_{k}^{\lambda}t}\label{eq:rate_spec}
\end{equation}
where the even number $p$ of frequency components, the frequencies
$\nu^\lambda = \{\nu_k^{\lambda} \}_k$ in a pre-specified band $[-B,+B]$, and the complex
coefficients $c^\lambda = \{c_k^\lambda \}_k$ are all unknown. Given the connections to Fourier series, this specification is very flexible and was introduced by Shao and Lii \citep{shao2010,shao2011}. They resolve the estimation problem under the classical setting where the frequencies are assumed to be spaced more than order $1/T$ apart. In this paper we examine the problem from the super-resolution perspective: We propose a simple procedure for estimating (\ref{eq:rate_spec}) when the frequencies can be up to order $1/T$ of each other. This is the finest possible resolution in the sense that no estimator can generally resolve frequencies separated by less than $1/T$ in the presence of noise \citep{moitra2015}.

Our approach modifies the classic periodogram and combines it with the super-resolution literature
on total-variation/atomic norm regularization. Three ingredients (to be specified
in Proposition \ref{prop:freqrecovery}) are used in Algorithm \ref{alg:proc}:
i) A window function $w(t)$ supported on $[0,T]$; ii) a threshold $\tau>0$; and iii) a neighbourhood exclusion radius $r>0$.
The simple but elegant intuition behind the thresholding idea (\citet{dj94})
as applied to our situation is that the spectral energy (given by
$|H(\nu)|$ as defined in the algorithm) should be concentrated at
the signal frequencies $\nu_{0}^{\lambda},\cdots,\nu_{p}^{\lambda}$.
If the signals are strong enough that $|H(\nu_{0}^{\lambda})|,\cdots,|H(\nu_{p}^{\lambda})|$
exceed the ambient noise level, then setting $\tau$ above the noise
level will result in the algorithm isolating a neighbourhood around
each $\nu_{k}^{\lambda}$ (see Figure \ref{fig:algointuition}). It
will be shown that if the frequencies are separated from one another
by a gap (resolution) of at least $g(T)/T$ where $g(T)\geq4$, then
with high probability our procedure will recover each $\nu_{k}^{\lambda}$
with a precision of $2/T$, provided that the dynamic range of the amplitudes $\max_{k}|c_{k}^{\lambda}|/\min_{k}|c_{k}^{\lambda}|$
is less than 14.5. This can be dramatically relaxed as the frequency gap is increased: For example if $g(T)\geq6$ then the maximum allowable dynamic range exceeds 100. As discussed in section 3 of \citep{shao2011}, some sort of dynamic range condition is needed even in the classical setting where the frequency gap is larger than order $1/T$. Our analysis provides a way for quantifying the maximum allowable range when $T$ is finite.

\begin{algorithm}

\caption{The proposed estimation procedure\label{alg:proc}}

\begin{description}
\item [{1}] Define the windowed periodogram for the point process as
\[
|H(\nu)|=\frac{1}{T}\left|\sum_{j}w(t_{j})e^{-2\pi i\nu t_{j}}\right|
\]
for $|\nu|\leq B$, and note that it is symmetric in $\nu$. The sum can be computed efficiently using non-uniform FFT algorithms like \citep{nufft}.
\item [{2}] Identify the frequency region $R=\{\nu:r\leq|\nu|\leq B,|H(\nu)| > \tau\}$
where the value of periodogram exceeds the threshold $\tau$.
\item [{3}] Set $\nu_0^\lambda = \hat{\nu}_0= 0$, $k=1$ and repeat the following steps:
\end{description}
\begin{itemize}
\item Find the highest stationary peak of the periodogram in $R$ and set
$\hat{\nu}_{k}$ as the corresponding frequency location. If no peaks
exist then exit loop.
\item Perform the updates $k\leftarrow k+1$ and $R\leftarrow R\backslash(\hat{\nu}_{k}-r,\hat{\nu}_{k}+r)$.
This removes a neighbourhood of radius $r$ centred at $\hat{\nu}_{k}$
from $R$.
\end{itemize}
\begin{description}
\item [{4}] Compute the estimator (\ref{eq:c-hat}) for $c_k^\lambda$.
\end{description}
\end{algorithm}

\begin{figure}
\begin{centering}
\includegraphics[bb=0bp 50bp 420bp 290bp,scale=0.5]{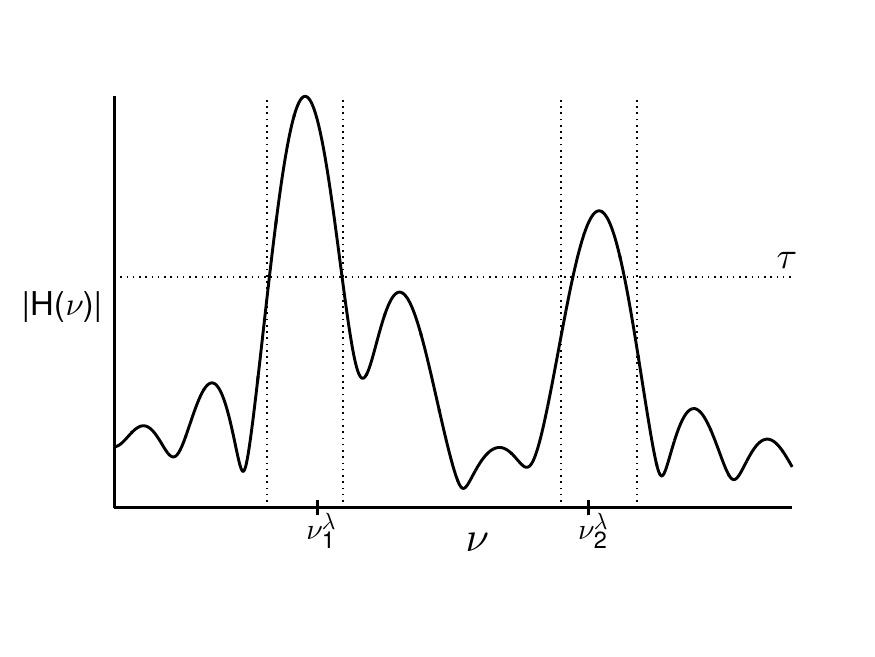}\caption{\label{fig:algointuition}Visualization of Algorithm \ref{alg:proc}.
In the depicted periodogram there are two signal frequencies $\nu_{1}^{\lambda}$
and $\nu_{2}^{\lambda}$. Setting $\tau$ (horizontal line) above
the ambient noise results in the algorithm selecting neighbourhoods
(between the pairs of vertical lines) that contain $\nu_{1}^{\lambda}$
and $\nu_{2}^{\lambda}$.}
\par\end{centering}
\end{figure}

A notable aspect of our methodology is in introducing the windowed periodogram to the point process domain: \citet{bartlett}'s classic `unwindowed' periodogram for point processes is essentially $\left|\sum_{j}e^{-2\pi i\nu t_{j}}\right|/T$, which is a special case of $|H(\nu)|$ when $w(t)$ is the rectangle window on $[0,T]$. We show that this window can and should be replaced with one that has faster decaying spectral tails. Doing so has two benefits. First, super-resolution can be achieved without needing to solve a semidefinite program. Second, even under the classical setting where the frequencies are spaced more than order $1/T$ apart ($g(T)\rightarrow\infty$ as $T\rightarrow\infty$), frequency estimation is more precise with a windowed periodogram. For example, Figure \ref{fig:convergence-rateintro} presents a log-log plot of the frequency estimation error versus $T$ for various choices of $g(T)$. The details of this experiment are elaborated upon in section \ref{subsec:freqrecovery}. The plots show that the rate of convergence increases with $g(T)$, with the windowed periodogram outperforming the unwindowed one until $g(T)$ reaches order $T^{1/2}$, whereupon both achieve the maximum rate of $\mathcal{O}(T^{-3/2})$ as predicted by theory.
\begin{figure}[h!]
\centering
\begin{subfigure}{0.47\textwidth}
\includegraphics[width=\textwidth]{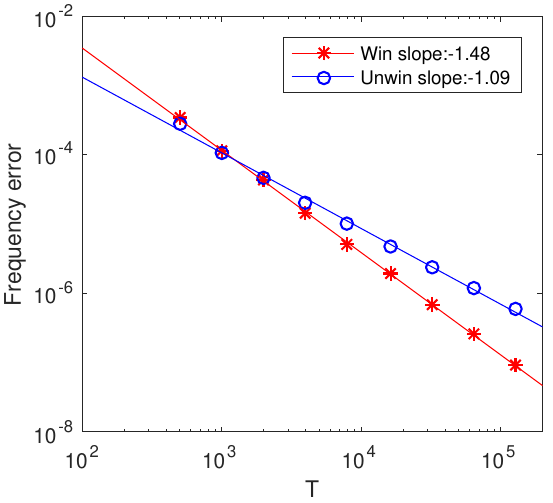}
\caption{$g(T)=6$.}
\end{subfigure}
\begin{subfigure}{0.47\textwidth} \includegraphics[width=1\textwidth]{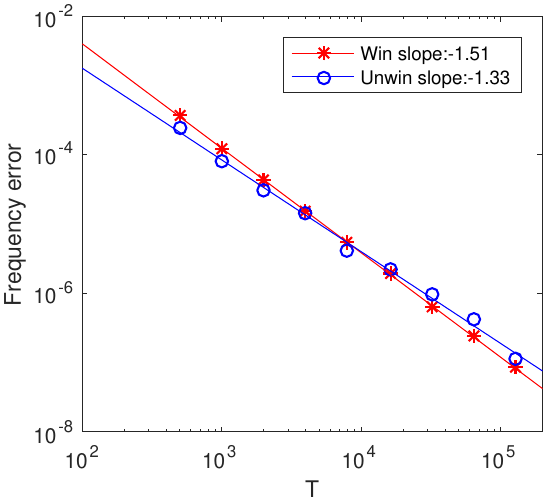}
\caption{$g(T)=T^{1/6}$.}
\end{subfigure} \begin{subfigure}{0.47\textwidth} \includegraphics[scale=0.65]{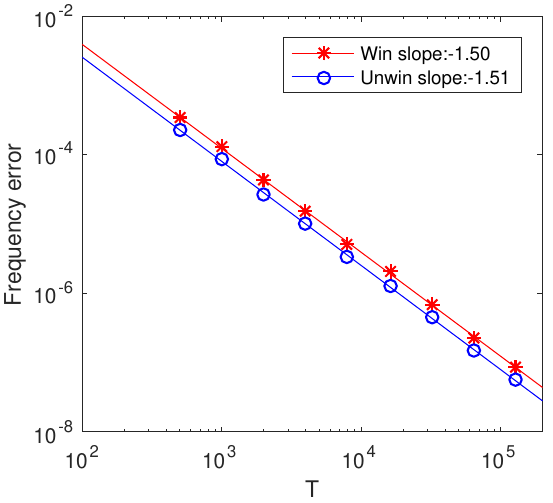}
\caption{$g(T)=T^{1/2}$.}
\end{subfigure} \caption{\label{fig:convergence-rateintro}Frequency recovery error $\max_k |\hat{\nu}_k-\nu^{\lambda}_k|$
for the simulation (\ref{eq:convergence-rate}) as a function of $T$, for $g(T) \in \{6, T^{1/6}, T^{1/2} \}$. `Win' refers to the windowed periodogram
and `Unwin' refers to the classic one. The error rates are $\max_k |\hat{\nu}_k-\nu^{\lambda}_k| \sim T^{s}$ where $s$ is the
slope of the relevant fitted line.}
\end{figure}

The remainder of this paper is organized as follows. Our contributions
to existing literature will be described below.
Section \ref{sec:overview} reviews some basic results from signal
processing and shows how spectral leakage and windowing manifest themselves
in arrivals data. This motivates the design of our estimation procedure. Frequency recovery is discussed in section \ref{sec:findfreq} where $w(t)$, $\tau$,
and $r$ are specified. Under the conditions given in the section,
it will be shown that our procedure will recover all frequencies
to within a precision of $2/T$ with high probability. We will also articulate the tradeoffs involved relative to methods designed for the classical resolution setting. The estimator for the corresponding amplitudes and phases is given in section \ref{sec:coefestimation}. In section \ref{sec:numerical} we use simulations to compare our procedure to the model selection approach in \citep{shao2010}. Concluding remarks can be found in section \ref{sec:discuss}.

\textbf{Contributions to literature.} Our estimation
procedure sits in between two streams of literature on spectral estimation. At one end, the classic approach is to visually inspect the unwindowed periodogram for point processes \citep{bartlett} to find frequencies corresponding to peaks in the plot. If it is assumed that there is only one frequency (\citet{lewis,vere-jones}), the frequency corresponding to the largest peak of the periodogram is selected. Other approaches \citep{bebb,belitser,helmers} also exist but it is unclear if they generalize to the setting with multiple frequencies.

Under the classical setting where the frequency gap is assumed to be $1/o(T)$, Shao and Lii \citep{shao2010,shao2011} extend the periodogram method to the multiple frequencies setting. Their procedure corresponds to setting $R=\{\nu: r\leq |\nu| \leq B\}$ in Algorithm \ref{alg:proc} and running step 3 until $p$ frequencies have been selected. The choice of $p$ is determined using the AIC/BIC model selection criterion derived in the dissertation of \citet{shao2010}. For BIC, \citep{shao2010} states that the probability of selecting the true $p$ eventually approaches 1 as $T\rightarrow \infty$. Our procedure builds on \citep{shao2010,shao2011} in two directions. First, the use of windowing enables periodogram methods to achieve super-resolution, and this can be combined with either thresholding or model selection to estimate $p$. Second, finite sample performance bounds can be derived for our thresholding approach. This complements the BIC approach which does not come with high probability guarantees for finite data. Moreover, the bounds also provide a way for quantifying the allowable amplitude dynamic range when $T$ is finite. As discussed in section 3 of \citep{shao2011}, some sort of dynamic range condition is needed even if the frequency gap is larger than order $1/T$. Of course, this will be more restrictive in the super-resolution setting, so there will be a cost to using our approach if the frequencies are in fact spaced far apart. This tradeoff will be discussed in section \ref{prop:freqrecovery}.

The other related stream of work is the super-resolution literature that uses total-variation or atomic norm regularization to select frequencies (\citet*{bhaskar,candes,fernandez2013,tang}). These papers study a generic spectral estimation problem in a discrete
time setting. They generalize the $\ell_{1}$-norm for a finite number
of variables to the case where there is a continuum of predictors
$\{e^{i\omega t}\}_{\omega}$. An infinite dimensional extension of Lasso is then formulated and
solved as a semidefinite program to select predictors and their coefficients. While the authors show that this method outperforms existing ones for the setting described, challenges arise when trying to adapt it to our problem. First, the arrival counts must be discretized into time bins, which introduces
aliasing effects\footnote{This can however be overcome using bins narrower than $1/(2B)$ (Nyquist sampling).}. Second, the required computational effort is overly taxing\footnote{A Lasso approximation obtained from discretizing the frequency domain
is suggested in \citep{bhaskar} as a speedup. However this is still
more difficult to implement than the periodogram method, along with
the additional downside of a fixed discretized frequency grid.} for the size of problems we consider. For example, \citep{CLS2018} analyzes 652 days of arrivals data from an emergency department and used 5,216 bins of 3 hour widths for the Lasso extension. The ADMM implementation recommended in \citep{bhaskar} takes at least ten days to run on a computer with Intel i7 6500 cores. By contrast our procedure takes only a few minutes.

In terms of frequency recovery, the approaches in \citep{fernandez2013,tang} are guaranteed to pick out one or more frequencies within some $C/T$ of each signal frequency when the resolution is $4/T$. In the stochastic noise setting of \citep{tang} the guarantee holds with high probability, and they further conjecture that it is possible to prevent the selection of spurious frequencies. We contribute to this literature by resolving the conjecture in the affirmative, since our procedure recovers exactly $p+1$ frequencies with high probability, one within $2/T$ of each true signal. The tradeoff with using a periodogram method is that a bound on the dynamic range of the amplitudes is needed. However as mentioned earlier, this can be dramatically relaxed by widening the frequency gap slightly, from $4/T$ to $6/T$ for example.

\section{Overview of the estimation approach\label{sec:overview}}

Let the continuum of complex exponentials $\left\{ e^{2\pi i\nu t}\right\} _{|\nu|\leq B}$
be our dictionary for constructing an arrival rate. Suppose the rate for the underlying
NHPP is (\ref{eq:rate_spec}), which belongs in the collection
\begin{equation}
\left\{ c_{0}+\sum_{k=1}^{p}c_{k}e^{2\pi i\nu_{k}t}:c_{k}\in\mathbb{C},p<\infty\right\} .\label{eq:coneA}
\end{equation}
Since $\lambda(t)$ is real-valued, (\ref{eq:rate_spec}) will lie
in the subset where the presence of $(c_{k},\nu_{k})$ implies its
conjugate $(\bar{c}_{k},-\nu_{k})$, so in particular $c_{0}$ will
be real and positive. The quantity of interest is the $(p+1)$-vector
$\nu^{\lambda}$ of frequencies in (\ref{eq:rate_spec}), where $p$
is even but unknown. Given these, the coefficients $c^{\lambda}$
in (\ref{eq:rate_spec}) will be estimated by the complex-valued least
squares solution (\ref{eq:c-hat}) described in section \ref{sec:coefestimation}.
Since $\lambda(t)$ is unobservable, we only see arrivals in the time
window $[0,T]$. Estimating the intensity therefore becomes a question
of recovering $\nu^\lambda$ from the frequency components in the trajectory $\{N(t)\}_{t\in[0,T]}$.
To make the connection between the spectrums of the two quantities
clearer, rewrite the latter in its Doob-Meyer form of signal and noise
components
\begin{equation}
\begin{aligned}\{dN(t)\}_{t\in[0,T]} & =[d\Lambda(t)+d\{N(t)-\Lambda(t)\}]I_{(0,T]}(t)\\
 & =\lambda(t)I_{(0,T]}(t)dt+d\varepsilon(t)I_{(0,T]}(t),
\end{aligned}
\label{eq:doobmeyer}
\end{equation}
where $I_{(0,T]}(t)$ is the indicator function of $\{0<t\leq T\}$.
Even in the absence of noise, the spectrum of the signal component
$\lambda(t)I_{(0,T]}(t)$ is itself a distorted version of the one
for $\lambda(t)$: Denoting the Fourier transform of $f(t)$ as
\[
\tilde{f}(\nu)=\int f(t)e^{-2\pi i\nu t}dt,
\]
we can write the spectrum of $\lambda(t)$ as the sum of the Dirac delta spikes centred at $\{\nu_k^\lambda\}_k$:
\[
\tilde{\lambda}(\nu)=\sum_{k=0}^{p}c_{k}^{\lambda}\delta(\nu-\nu_{k}^{\lambda}).
\]
On the other hand $\lambda(t)I_{(0,T]}(t)$
is the result of truncating $\lambda(t)$ due to $T$ being finite,
a spectrum distorting operation known as leakage: Denote the convolution
operator $\ast$ by $f\ast h(t)=\int f(s)h(t-s)ds$, the $h$-smoothed
average of $f$ about the point $t$. The spectrum of $\lambda(t)I_{(0,T]}(t)$ is
\begin{equation}
\widetilde{(\lambda\cdot I_{(0,T]})}(\nu)=\left(\tilde{\lambda}\ast\tilde{I}_{(0,T]}\right)(\nu)=\sum_{k=0}^{p}c_{k}^{\lambda}\tilde{I}_{(0,T]}(\nu-\nu_{k}^{\lambda}),\label{eq:leakage}
\end{equation}
a weighted average of $\lambda(t)$'s spectral values $c_{k}^{\lambda}$
concentrated at $\{\nu_k^\lambda\}_k$. Thus truncation has the effect of smearing
the frequency spikes in $\tilde{\lambda}(\nu)$ into a continuous
spectrum: For $\nu \notin \cup_k\{\nu_k^\lambda\}$,
$\widetilde{\lambda\cdot I_{(0,T]}}(\nu)$ can have a non-zero value,
creating an artificial noise floor. The noise floor around strong
signal frequencies may mask weaker neighbouring signals, leading to
resolution loss and making it difficult to recover $\nu^\lambda$
from $\lambda(t)I_{(0,T]}(t)$. Leakage distortion is a manifestation
of the uncertainty principle because perfect frequency localization
requires $\tilde{I}_{(0,T]}(\nu)=\delta(\nu)$, but this is only possible
if $I_{(0,T]}(t)=1$, i.e. an infinite time window
is needed.

The key idea that Algorithm \ref{alg:proc} uses to deal with leakage
is to replace $I_{(0,T]}(t)$ with a suitably chosen window function $w(t)$ to obtain
the weighted arrival process $dN^{w}(t)=w(t)dN(t)$. We see from (\ref{eq:leakage}) that
the extent of leakage depends on the tail decay of $\tilde{I}_{(0,T]}$,
as this dictates the influence that distant frequencies has on the
local spectral value. Since $\lambda(t)$ can be truncated to $(0,T]$
using any $w(t)$ supported on $(0,T]$, we can multiply $\lambda(t)$ with one whose Fourier
transform has lighter tails.

While the usual anti-leakage benefits of non-uniform windows is well known in signal processing, they are in fact needed in our procedure for attaining frequency resolutions of order $1/T$: The tail decay of $\tilde{I}_{(0,T]}(\nu)$ is of order $1/(T\nu)$. Thus if $\{\nu_k^\lambda\}_k$ are spaced $1/T$ apart, the leakage (\ref{eq:leakage}) around a neighbourhood of $\nu_{k}^{\lambda}$ from the other frequencies can be of order $\log p$ for the rectangle window. This can easily mask the periodogram spike at $\nu_{k}^{\lambda}$ when $p$ is large enough. Hence the classic periodogram method is generally unable to attain frequency resolutions of order $1/T$. Interestingly the window that is usually considered optimal for signal processing\footnote{Optimal in the sense that its spectrum is the one that is most concentrated about the origin.} is actually suboptimal for frequency recovery: Theorem 3.44 of \citep{osipov}
shows that the spectral tail decay of the prolate spheroidal function is also of order $1/(T\nu)$ when it is time-limited to $[0,T]$.

Returning to the problem of recovering $\nu^\lambda$
from (\ref{eq:doobmeyer}), consider the $(1/T)$-scaled spectrum
of the windowed data $dN^{w}(t)=w(t)dN(t)$:
\begin{equation}
\begin{aligned}H(\nu) & =\frac{1}{T}\int_0^T e^{-2\pi i\nu t}dN^{w}(t)\\
 & =\frac{1}{T}\int_0^T e^{-2\pi i\nu t}w(t)\lambda(t)dt+\frac{1}{T}\int_0^T e^{-2\pi i\nu t}w(t)d\varepsilon(t)\\
 & =\frac{1}{T}\sum_{k=0}^{p}c_{k}^{\lambda}\tilde{w}(\nu-\nu_{k}^{\lambda})+\frac{\tilde{\varepsilon}^{w}(\nu)}{T}.
\end{aligned}
\label{eq:periodogram}
\end{equation}
Recall from Algorithm \ref{alg:proc} that $|H(\nu)|$ is defined as the windowed
periodogram. For $\nu$ sufficiently far from $\{\nu_k^\lambda\}_k$,
the noise level outside the vicinity of these frequencies should be
low for light tailed $\tilde{w}$:
\begin{equation}
|H(\nu)|\leq \|c^{\lambda}\|_\infty \sum_{k=0}^{p}\frac{|\tilde{w}(\nu-\nu_{k}^{\lambda})|}{T}+\sup_{\nu\in[0,B]}\frac{|\tilde{\varepsilon}^{w}(\nu)|}{T}.\label{eq:leakagefloor}
\end{equation}
If the signal strengths $c^\lambda$ are sufficiently strong, then intuitively a neighbourhood of $\cup_{k}\{\nu_{k}^{\lambda}\}$
can be isolated by simply excluding frequency regions in $[-B,+B]$
where $|H(\nu)|$ is below some threshold $\tau$ (see Figure \ref{fig:algointuition}).
This is the idea behind step 2 of Algorithm \ref{alg:proc}. The analysis
presented in the next section will guide our choices for $w(t)$,
$\tau$, and $r$ in our estimation procedure.

\section{Frequency recovery\label{sec:findfreq}}

To guarantee that Algorithm \ref{alg:proc} will recover the true
signal frequencies $\nu^{\lambda}$ with high probability, we will assume that conditions A1 and A2 given in this section hold from the point they are stated. First, since no method can distinguish among frequencies that are clustered arbitrarily close together, we impose a minimum separation gap.
\begin{description}
\item [{A1}] For $0\leq k,k'\leq p$, $\min_{k\neq k'}|\nu_{k}^{\lambda}-\nu_{k'}^{\lambda}|\geq \frac{g(T)}{T}$
for some $g(T)\geq4$.
\end{description}
The gap $g(T)/T$ represents the frequency resolution for our procedure, and our recovery results cover all possible rates of growth for $g(T)$ as $T\rightarrow\infty$. The lower bound of $4/T$ benchmarks the frequency gap employed in the super-resolution literature \citep{fernandez2013,tang}. If instead the benchmark target is the classical setting in \citep{shao2010,shao2011}, then A1 may be relaxed to $6/T$, see the remark following Proposition \ref{prop:freqrecovery} below.

Under A1, we must localize each $\nu_{k}^{\lambda}$ to within a neighbourhood
of radius $2/T$ to avoid possible ambiguity from overlapping. To
achieve this with thresholding, note from (\ref{eq:leakagefloor})
that if $\nu$ is at least $2/T$ away from the nearest $\nu_{k}^{\lambda}$,
then $|H(\nu)|$ is strictly less than\footnote{The sum to infinity is needed as $p$ is unknown.}
\begin{equation}
\underbrace{\frac{2}{T}\sum_{l=0}^{\infty}\sup_{|\nu|\geq\frac{2}{T}+\frac{4}{T}l}|\tilde{w}(\nu)|}_{S_{1}}\cdot\|c^{\lambda}\|_{\infty}+\sup_{\nu\in[0,B]}\frac{|\tilde{\varepsilon}^{w}(\nu)|}{T},\label{eq:pre-tau}
\end{equation}
where the tail sum $S_{1}$ bounds the leakage noise floor outside
the vicinity of $\{\nu_k^{\lambda}\}_k$,
and the last term is the statistical noise level. The unknown $\|c^{\lambda}\|_{\infty}$
can be estimated using the highest peak of the periodogram: It is
shown in Appendix \ref{appendix:proofs} that
\begin{equation}
\left(\frac{|\tilde{w}(0)|}{T}-\underbrace{\frac{2}{T}\sum_{l=1}^{\infty}\sup_{|\nu|\geq\frac{4}{T}l}|\tilde{w}(\nu)|}_{S_{2}}\right)\|c^{\lambda}\|_{\infty}-\sup_{\nu\in[0,B]}\frac{|\tilde{\varepsilon}^{w}(\nu)|}{T}\leq\sup_{\nu\in[0,B]}|H(\nu)|,\label{eq:cmaxUB}
\end{equation}
\begin{equation}
\sup_{\nu\in[0,B]}|H(\nu)|\leq\max\left(S_{1},\frac{|\tilde{w}(0)|}{T}+\frac{S_{1}+S_{2}}{2}\right)\|c^{\lambda}\|_{\infty}+\sup_{\nu\in[0,B]}\frac{|\tilde{\varepsilon}^{w}(\nu)|}{T}.\label{eq:cmaxLB}
\end{equation}
Substituting the bound (\ref{eq:cmaxUB}) for $\|c^\lambda\|_\infty$ into (\ref{eq:pre-tau}) shows that the threshold level $\tau$ in Algorithm \ref{alg:proc} should be
\begin{equation}
\frac{S_{1}}{|\tilde{w}(0)|/T-S_{2}}\sup_{\nu\in[0,B]}|H(\nu)|+\left(\frac{S_{1}}{|\tilde{w}(0)|/T-S_{2}}+1\right)\sup_{\nu\in[0,B]}\frac{|\tilde{\varepsilon}^{w}(\nu)|}{T}\label{eq:tau}
\end{equation}
in order to remove from the region $R$ all frequencies not within
$2/T$ of any $\nu_{k}^{\lambda}$. Our procedure will then select
a unique frequency within $2/T$ of each $\nu_{k}^{\lambda}$ if $|H(\nu_{k}^{\lambda})| > \tau$,
so we can set $r=2/T$. In view of (\ref{eq:periodogram}) and (\ref{eq:cmaxLB}),
a sufficient condition for $|H(\nu_{k}^{\lambda})| > \tau$ is
\[
\begin{aligned} & \frac{|\tilde{w}(0)|}{T}|c_{k}^{\lambda}|-S_{2}\|c^{\lambda}\|_{\infty}-\sup_{\nu\in[0,B]}\frac{|\tilde{\varepsilon}^{w}(\nu)|}{T}\\
> & \frac{S_{1}}{|\tilde{w}(0)|/T-S_{2}}\left\{ \max\left(S_{1},\frac{|\tilde{w}(0)|}{T}+\frac{S_{1}+S_{2}}{2}\right)\|c^{\lambda}\|_{\infty}+\sup_{\nu\in[0,B]}\frac{|\tilde{\varepsilon}^{w}(\nu)|}{T}\right\} \\
+ & \left(\frac{S_{1}}{|\tilde{w}(0)|/T-S_{2}}+1\right)\sup_{\nu\in[0,B]}\frac{|\tilde{\varepsilon}^{w}(\nu)|}{T}
\end{aligned}
\]
for $k=0,\cdots,p$, or equivalently
\begin{equation}
\begin{aligned}\frac{|\tilde{w}(0)|}{T}\min_{k}|c_{k}^{\lambda}| & > \left\{S_{2}+\frac{S_{1}\max\left( S_{1},\frac{|\tilde{w}(0)|}{T}+\frac{S_{1}+S_{2}}{2}\right) }{|\tilde{w}(0)|/T-S_{2}}\right\}\max_{k}|c_{k}^{\lambda}|\\
 & +2\left(\frac{S_{1}}{|\tilde{w}(0)|/T-S_{2}}+1\right)\sup_{\nu\in[0,B]}\frac{|\tilde{\varepsilon}^{w}(\nu)|}{T}.
\end{aligned}
\label{eq:coefsizes}
\end{equation}
It will be shown that the first two terms are dominant. Hence to first
order, as the tail sums $S_{1}$ and $S_{2}$ become small relative
to $|\tilde{w}(0)|/T$, a larger margin of separation between signal
and leakage noise is attained in frequency domain. Therefore window functions
with rapidly decaying spectral tails are desired. Of the commonly
used continuous time windows presented in Table 3.1 of \citet{prabhu}
with spectral energy concentrated inside $|\nu|<2/T$, the time-shifted
Hann window has the lightest spectral tails (order $1/(T\nu)^{3}$):
\begin{equation}
w(t)=\left(\sin^{2}\frac{\pi t}{T}\right)I_{[0,T]}(t)\leftrightarrow\tilde{w}(\nu)=\begin{cases}
T/2 & \nu=0\\
-T/4 & \nu=\pm\frac{1}{T}\\
\frac{T}{2}e^{-i\pi T\nu}\frac{\sinc(T\nu)}{1-(T\nu)^{2}} & \mbox{else}
\end{cases},\label{eq:hann}
\end{equation}
where $\sinc(\nu)=\sin(\pi\nu)/(\pi\nu)$ is the sinc kernel. Note
from Figure \ref{fig:hann} that $|\tilde{w}(\nu)|$ is symmetric
and most of its energy is concentrated inside the main lobe between
$\nu=\pm\frac{2}{T}$. The sidelobes are of width $1/T$ and have
successively lower peaks. The following lemma provides estimates for
$S_{1}$ and $S_{2}$.

\begin{figure}
\centering{}\includegraphics[bb=100bp 50bp 819bp 285bp,clip,scale=0.4]{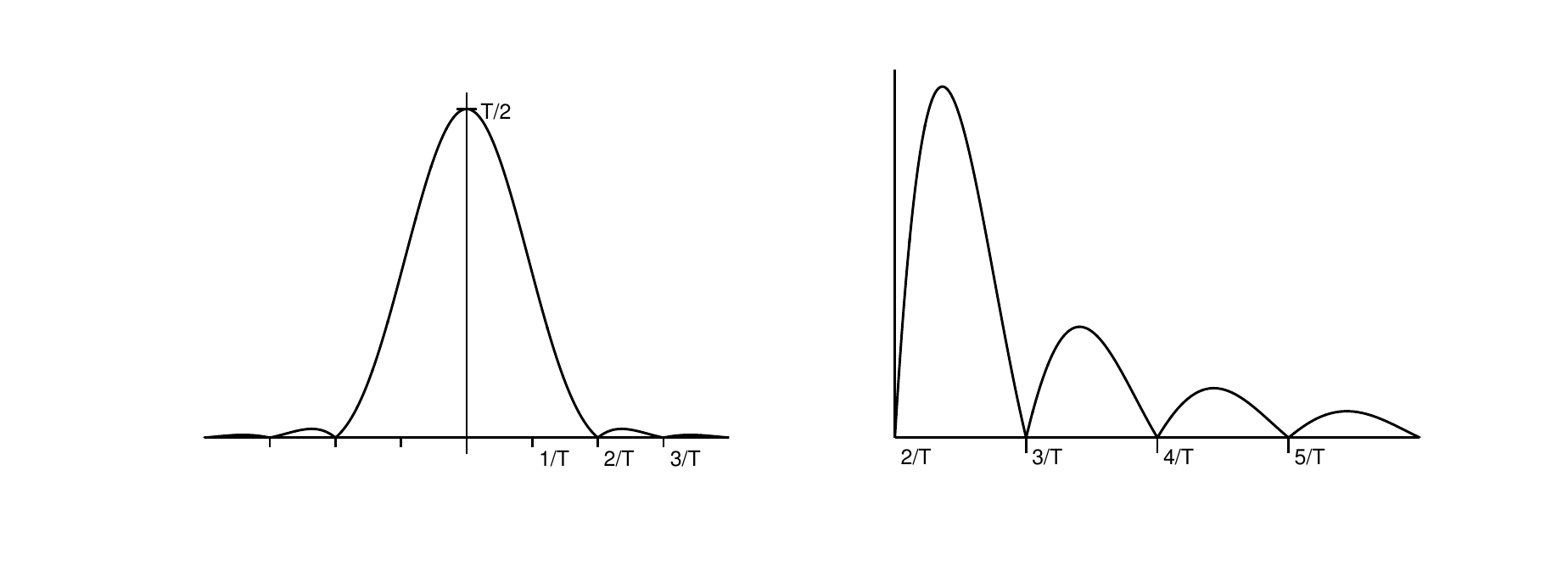}\caption{\label{fig:hann}Plot of $|\tilde{w}(\nu)|$ for the Hann window (\ref{eq:hann}).
\textit{Left panel:} Most of the energy is concentrated in the main
lobe between $\nu=\pm\frac{2}{T}$. \textit{Right panel:} The side
lobes are of width $1/T$ and have successively lower peaks.}
\end{figure}

\begin{lem} \label{lem:hann_tailsum}For the Hann window
\[
0.02843<S_{1}=\frac{2}{T}\sum_{l=0}^{\infty}\sup_{|\nu|\geq\frac{2}{T}+\frac{4}{T}l}|\tilde{w}(\nu)|<0.02844,
\]
\[
0.00464<S_{2}=\frac{2}{T}\sum_{l=1}^{\infty}\sup_{|\nu|\geq\frac{4}{T}l}|\tilde{w}(\nu)|<0.00465.
\]
Furthermore if we define $\tilde{w}'(\nu) = \frac{d\tilde{w}}{d\nu}(\nu)$, then for any $\nu\in(\nu_{k}^{\lambda}-\frac{2}{T},\nu_{k}^{\lambda}+\frac{2}{T})$,
\[
\frac{1}{T}\sum_{l\neq k}|\tilde{w}(\nu-\nu_{l}^{\lambda})|< \frac{4}{g(T)^3}, \, \frac{1}{T}\sum_{l\neq k}|\tilde{w}'(\nu-\nu_{l}^{\lambda})|< \frac{29T}{g(T)^3}.
\]
\end{lem}

The remaining quantity not yet examined in (\ref{eq:tau}) and (\ref{eq:coefsizes})
is the supremum spectral density $\sup_{\nu\in[0,B]}|\tilde{\varepsilon}^{w}(\nu)|$
of the windowed statistical noise. Noting that $|w(t)|\leq1$ and
$|w'(t)|=\left|\frac{dw}{dt}(t) \right|\leq\pi/T<\infty$, the following lemma shows that
the scaled spectral noise level is of order $(\log T/T)^{1/2}$.

\begin{lem}\label{lem:dualnorm}Define $\bar{\Lambda}_{T}=\Lambda(T)/T$
and $\bar{N}_{T}=N(T)/T$, and suppose that $\sup_{t\in[0,T]}|w(t)|\leq1$,
$\sup_{t\in(0,T)}|w'(t)|<\infty$. Then for any $\beta>0$, $\gamma>1$,
and $\alpha\geq\gamma/(\gamma-1)$, with probability
\[
1-8\gamma\pi B\left[1/T^{\left(\frac{\gamma-1}{\gamma}\alpha\right)^{2}-1} + T\exp\left\{ -(\Lambda(T)\log T)^{1/2}\right\} \right]-2e^{-\Lambda(T)\beta^{2}/4}
\]
we have
\[
(1-\beta)\bar{\Lambda}_{T}<\bar{N}_{T}<(1+\beta)\bar{\Lambda}_{T}
\]
and
\[
\sup_{\nu\in[0,B]}\frac{|\tilde{\varepsilon}^{w}(\nu)|}{T}<4\alpha\bar{\Lambda}_{T}^{1/2}\left(\frac{\log T}{T}\right)^{1/2}<\frac{4\alpha\bar{N}_{T}^{1/2}}{(1-\beta)^{1/2}}\left(\frac{\log T}{T}\right)^{1/2}.
\]
\end{lem}



Lemmas \ref{lem:hann_tailsum} and \ref{lem:dualnorm} can be used in (\ref{eq:tau}) to define the data-driven threshold
\begin{equation}
\tau=0.0574\sup_{\nu\in[0,B]}|H(\nu)|+\frac{4.23\alpha\bar{N}_{T}^{1/2}}{(1-\beta)^{1/2}}\left(\frac{\log T}{T}\right)^{1/2}\label{eq:hann_tau}
\end{equation}
for the Hann window. Deriving the sufficient condition for frequency recovery (\ref{eq:coefsizes}) for this $\tau$ and the Hann window yields:
\begin{description}
\item [{A2}] There exist $\beta>0$, $\gamma>1$, and $\alpha\geq\gamma/(\gamma-1)$
such that
\[
\min_{k}|c_{k}^{\lambda}| > 0.0686\max_{k}|c_{k}^{\lambda}|+16.9\alpha\left\{ 1+\left(\frac{1+\beta}{1-\beta}\right)^{1/2}\right\} \bar{\Lambda}_{T}^{1/2}\left(\frac{\log T}{T}\right)^{1/2}.
\]
\end{description}
As $T$ grows the last term in A2 vanishes, so to first order the
condition $\min_{k}|c_{k}^{\lambda}| > 0.0686\max_{k}|c_{k}^{\lambda}|$
requires the dynamic range of the amplitudes to be less than 14.5.
The smaller the tailsums $S_{1}$ and $S_{2}$ are, the larger the
allowable range. In particular if the gap in A1 is slightly relaxed
from $4/T$ to $6/T$, the value of 14.5 can be increased to over
100 by replacing the Hann window with the lighter spectral-tailed
$\cos^{4}$ window \citep{prabhu}. Thus windows with light spectral
tails provide a solution for detecting weak frequency signals in the
presence of strong ones. This addresses a point mentioned in passing
on page 110 of \citep{shao2011}: Issues with the periodogram method arise when the
dynamic range is large, even in the classical setting where the frequency gap is
$1/o(T)$. Our analysis provides a way for quantifying this for both the windowed and unwindowed periodograms when $T$ is finite. In the special case where all the frequencies have the same amplitude $|c_1^\lambda| = \cdots = |c_p^\lambda|$, A2 simplifies to requiring the amplitude to be larger than a multiple of the statistical noise level (last term of A2).

The main frequency recovery result can now be stated under A1 and A2.

\begin{prop} \label{prop:freqrecovery}Let $w(t)$ in Algorithm \ref{alg:proc}
be the Hann window (\ref{eq:hann}), and set $r=2/T$ and $\tau$
as (\ref{eq:hann_tau}). Then with probability at least
\[
1-8\gamma\pi B\left[1/T^{\left(\frac{\gamma-1}{\gamma}\alpha\right)^{2}-1} + T\exp\left\{ -(\Lambda(T)\log T)^{1/2}\right\} \right]-2e^{-\Lambda(T)\beta^{2}/4}
\]
our procedure will select exactly $p+1$ frequencies $\hat{\nu}=\{\hat{\nu}_{k}\}_{k}$
with precision $\|\nu^{\lambda}-\hat{\nu}\|_{\infty} < 2/T$. Furthermore,
\begin{equation*}
\|\nu^{\lambda}-\hat{\nu}\|_{\infty} < \min \left\{ \frac{2}{T} , \frac{2\epsilon(T)}{T} \right\}
\end{equation*}
if
\begin{equation*}
  \epsilon(T)\triangleq\frac{348 \left( \|c^{\lambda}\|_{\infty} +
    \alpha \bar{\Lambda}_T^{1/2} \right) }{\min_{k}|c_{k}^{\lambda}|} \max \left\{ \frac{1}{g(T)^3}, \left( \frac{\log T}{T} \right)^{1/2} \right\}\le  \frac{87}{40}.
\end{equation*}
\end{prop}

\begin{rem*} Through the use of windowing, we obtain the first periodogram peak-hunting method that is able to achieve super-resolution. Note from the definition of $\epsilon(T)$ that if $g(T)\rightarrow \infty$ then the procedure will recover all frequencies with precision $o(1/T)$. In particular, if $g(T)$ is $\mathcal{O}(T^{1/6})$ or greater then the estimation error is $\mathcal{O}(T^{-3/2})$ up to a log factor. For the unwindowed periodogram in the closely related time series setting, Theorem 6.8b of \citet{li2014} shows that the same rate is achieved when $g(T)$ is greater than $\mathcal{O}(T^{1/2})$. This is because $T^{3/2}(\nu^\lambda-\hat\nu)$ has a bias of $\mathcal{O}(T^{1/2}/g(T))$ due to the slower spectral tail decay of the rectangle window (Remark 6.14 of \citep{li2014}). Thus even under the classical resolution setting, windowing is still beneficial since it sharpens the precision of the frequency estimates. 
\end{rem*}

\begin{rem*}
  In applications, $\alpha$, $\beta$, and $\gamma$ are chosen to balance a number of
  considerations. First is the expected dynamic range (A2) for the particular
  problem being considered. Second is the bandwidth $B$: If A2 holds for values of $\alpha,\gamma$ satisfying
  $\alpha(\gamma-1)/\gamma \geq \sqrt{2}$, then the probability bound above is $1-8\gamma\pi
  B/T - 2 e^{{-\Lambda(T)\beta^{2}/4}}$ to leading order in $B/T$, in which case $B$ has the same asymptotic scaling as \citep{shao2010,shao2011,vere-jones}. Third is the desired recovery
  probability. One possible choice that balances these considerations is
  $\alpha=2$, $\beta = 2 \sqrt{\log T /T }$, and $\gamma = 4$, which simplifies the probability bound to
\begin{equation*}
  1-32\pi B\left[T^{-5/4} + T\exp\left\{ -(\Lambda(T)\log T)^{1/2}\right\} \right]-2T^{-\bar{\Lambda}_{T}}.
\end{equation*}
\end{rem*}

\begin{rem*}
When $p$ is known, no thresholding is necessary, and the asymptotic normality results in \citep{shao2011} for the classic periodogram can be extended to the windowed one. The details are provided in Appendix \ref{appendix:normality}.
\end{rem*}

\vspace{0.2cm}
\textbf{When does the approach of \citep{shao2010,shao2011} perform better?} If $\{\nu^\lambda_k\}_k$ are in fact spaced more than order $1/T$ apart from one another, then it follows from (\ref{eq:leakagefloor}) that the leakage outside a $\mathcal{O}(g(T)/T)$-neighbourhood of the frequencies is of order $1/g(T)^3 \rightarrow 0$ for the Hann-windowed periodogram. Hence the threshold (\ref{eq:hann_tau}) is conservative in this setting. While it will still work within the dynamic range implied by A2, we expect the method in \citep{shao2010,shao2011}, which was specifically designed for the classical resolution setting, to recover more of the frequencies with amplitude less than $1/14.5$ of the largest one. Of course, the Hann window analyzed here can also be used with the method in \citep{shao2010,shao2011}.

\vspace{0.2cm}
\textbf{Connection to super-resolution literature.} There are clear connections between our results and those arising
from the work on super-resolution recovery of discrete time signals
\citep{bhaskar,candes,fernandez2013,tang}. In that setting the authors
assume a discrete time signal $x=\sum_{k=1}^{p}c_{k}^{\lambda}e^{2\pi i\nu_{k}t}\in\reals^{n}$
and the observations are of the form $y=x+e$ where $e\in\reals^{n}$
is a noise vector.

For a bounded $e$, \citep{candes,fernandez2013} establish signal
and support recovery guarantees for their semidefinite programming
approach. On the other hand for $e_{i}\sim N(0,\sigma^{2})$
the related AST approach \citep{bhaskar,tang} achieves near minimax
rates. Furthermore if $\min_{k}|c_{k}^{\lambda}|$ is larger than
some multiple of $\sigma p(\log n/n)^{1/2}$, then with high probability
AST is guaranteed to pick out one or more frequencies within some
$C/n$ of each signal frequency. The authors conjecture that it is
possible to prevent the selection of spurious frequencies, and that
the sparsity $p$ can be dropped from the lower bound on $\min_{k}|c_{k}^{\lambda}|$.
The following corollary shows that our procedure resolves these conjectures
in the affirmative when applied to this setting.

\begin{cors} Suppose $T$ is replaced by $n$ and $\bar{\Lambda}_{T}$
is replaced by $\sigma^{2}$ in A2. Under the discrete time setting
above, with high probability our procedure will select exactly $p$
frequencies within distance $\|\widehat{\nu}-\nu^{\lambda}\|_{\infty}\leq4/n$
of the true ones. \end{cors}

\vspace{0.2cm}
\textbf{Modified threshold.} The last term in the threshold (\ref{eq:hann_tau})
comes from the spectral noise bound in Lemma \ref{lem:dualnorm},
whose constant $4\bar{N}_{T}^{1/2}$ may be conservative. As a result
we observe in experiments that a large value of $T$ is sometimes
needed for the guarantees to hold with high probability. To obtain
a tighter estimate, one idea is to approximate the spectral noise
level of the underlying nonhomogeneous Poisson process with that of
a homogeneous one. This is motivated by the fact that the noise bound
in Lemma \ref{lem:dualnorm} depends on $\lambda(t)$ only through
the average rate $\bar{N}_{T}$, regardless of whether the Poisson
process is homogeneous or not. Thus for a given $\xi>0$, consider
the modified threshold
\begin{equation}
\tau_{\xi}=(0.0574+\xi)\sup_{\nu\in[0,B]}|H(\nu)|+1.06\min\left\{ \hat{\chi}_{T},\frac{4\alpha\bar{N}_{T}^{1/2}}{(1-\beta)^{1/2}}\left(\frac{\log T}{T}\right)^{1/2}\right\} \label{eq:mod_hann_tau}
\end{equation}
where $\hat{\chi}_{T}$ is the simulated $\sup_{\nu\in[0,B]}|\tilde{\varepsilon}^{w}(\nu)|/T$ for the homogeneous Poisson process with rate $\bar{N}_{T}$ over $[0,T]$. It is equivalent to applying the expression (\ref{eq:cperiodogram}) to simulated data. Clearly, if the second quantity in the
curly bracket is smaller then we effectively recover (\ref{eq:hann_tau}).
In experiments we find that thresholding with $\tau_{\xi}$ performs better than $\tau$ in practice. The following corollary provides a large sample recovery guarantee for $\tau_\xi$.
\begin{cors}\label{cor:mod_tau}
Suppose A2 is slightly strengthened to $\min_{k}|c_{k}^{\lambda}|>(0.0686+4\xi)\max_{k}|c_{k}^{\lambda}|$,
and $T$ is large enough that $\alpha\left(\frac{1+\beta}{1-\beta}\cdot\frac{\bar{\Lambda}_{T}\log T}{T}\right)^{1/2}\leq\frac{30\xi}{28+25\xi}\|c^{\lambda}\|_{\infty}$.
Then with probability at least
\[
1-8\gamma\pi B\left[1/T^{\left(\frac{\gamma-1}{\gamma}\alpha\right)^{2}-1} + T\exp\left\{ -(\Lambda(T)\log T)^{1/2}\right\} \right]-2e^{-\Lambda(T)\beta^{2}/4}
\]
all frequencies will be recovered with the precision stated in Proposition
\ref{prop:freqrecovery} when we threshold with $\tau_{\xi}$.
\end{cors}
\begin{rem*}
	If the second condition is to ever hold, $\xi$ must then be at least of order $\|c^\lambda\|_\infty^{-1}\sqrt{\bar{\Lambda}_T \log T/T}$.
\end{rem*}

\section{Amplitude and phase estimation\label{sec:coefestimation}}

As noted by \citet{rice} for the case of cyclic time series and \citep{shao2010,shao2011}
for the case of cyclic Poisson processes, it is necessary for the
estimated frequencies $\hat{\nu}$ to be within $o(1/T)$ of $\nu^{\lambda}$
if we wish to estimate the coefficients $c^{\lambda}$ consistently.
We will therefore let $g(T)\rightarrow\infty$ in Proposition \ref{prop:freqrecovery}
so that $\epsilon(T)\rightarrow0$. Our estimator is the complex-valued
least squares solution to (\ref{eq:doobmeyer}) in the limit $dt\rightarrow0$:
\begin{equation}
\hat{c}=\hat{\Gamma}^{-1}y,\label{eq:c-hat}
\end{equation}
where the $j$-th entry of the $(p+1)$-vector $y$ is $\frac{1}{T}\int_{0}^{T}e^{-2\pi i\hat{\nu}_{j}t}dN(t)$,
and the $(j,k)$-entry of the $(p+1)\times(p+1)$ matrix $\hat{\Gamma}$
is
\[
\hat{\Gamma}_{jk}=\frac{1}{T}\int_{0}^{T}e^{-2\pi i(\hat{\nu}_{j}-\hat{\nu}_{k})t}dt=\frac{1}{T}\tilde{I}_{(0,T]}(\hat{\nu}_{j}-\hat{\nu}_{k}),
\]
where $\tilde{I}_{(0,T]}(\nu)=Te^{-i\pi T\nu}\sinc(T\nu)$ is the
Fourier transform of the rectangle. Since $\{\hat{\nu}_k\}_k$
are symmetric about zero, it can be shown for $\hat{\nu}_{k}=-\hat{\nu}_{l}$
that $\hat{c}_{k}$ and $\hat{c}_{l}$ are conjugate pairs, hence
the estimator for $\lambda(t)$ is always real-valued. We note that
the corresponding estimator in \citep{shao2010,shao2011} can be recovered by setting $\hat{\Gamma}$
to the identity matrix, which is asymptotically valid because $\hat{\Gamma}$
converges to an orthonormal design as $g(T)\rightarrow\infty$. Our
choice of $\hat{\Gamma}$ provides a second order correction when $T$ is finite.

\begin{prop} \label{prop:LSerror}Suppose the conditions for Proposition
\ref{prop:freqrecovery} hold with $\epsilon(T)\leq 87/40$, and that
\[
\Gamma_{jk}=\frac{1}{T}\int_{0}^{T}e^{-2\pi i(\nu_{j}^{\lambda}-\nu_{k}^{\lambda})t}dt=\frac{1}{T}\tilde{I}_{(0,T]}(\nu_{j}^{\lambda}-\nu_{k}^{\lambda})
\]
is invertible. Then with probability at least
\[
1-8\gamma\pi B\left[1/T^{\left(\frac{\gamma-1}{\gamma}\alpha\right)^{2}-1} + T\exp\left\{ -(\Lambda(T)\log T)^{1/2}\right\} \right]-2e^{-\Lambda(T)\beta^{2}/4}
\]
i) $\hat{\Gamma}$ is also invertible for sufficiently large $T$;
and ii)
\[
\|\hat{c}-c^{\lambda}\|_{\infty}<2\left\{ (\pi+2\alpha)\|\hat{\Gamma}^{-1}\|\max(\|c^{\lambda}\|_{1},1)\right\} \epsilon(T).
\]
\end{prop}


\section{Numerical examples\label{sec:numerical}}

We use simulations to compare our thresholding procedure (based on the modified threshold)
to the windowed periodogram combined with BIC model selection, and also to the classic periodogram in \citep{shao2010,shao2011} combined with BIC. We also use our procedure to analyze arrivals data from an academic emergency department in the United States. We focus on the BIC because it is asymptotically consistent,
and the corresponding penalized log-likelihood for Poisson processes is derived in section 3.3.4 of \citep{shao2010}:
\begin{equation}\label{eq:BIC}
-2\left( \sum_{j=1}^{N(T)} \log \lambda(t_j) - \Lambda(T) \right) + (5p+1)\log T.
\end{equation}
The algorithm for using the windowed periodogram with BIC selection corresponds to setting $R=\{\nu: r\leq |\nu| \leq B\}$ in Algorithm \ref{alg:proc} and running step 3 until $p$ frequencies have been selected. The value of $p$ is chosen to minimize (\ref{eq:BIC}).

Since by default the frequency $\nu=0$ is always selected, we work
with the centralized version of $|H(\nu)|$ instead:
\begin{equation}
\begin{aligned}|H_{c}(\nu)| & =\frac{1}{T}\left|\int_{0}^{T}e^{-2\pi i\nu t}w(t)\left(dN(t)-\frac{N(T)}{T}dt\right)\right|\\
 & =\frac{1}{T}\left|H(\nu)-\frac{N(T)}{T}\tilde{w}(\nu)\right|,
\end{aligned}
\label{eq:cperiodogram}
\end{equation}
which is one way to generalize the centralized unwindowed periodogram
given by equation 4 in \citep{shao2011}. This approximately removes
from the windowed periodogram the peak at the origin.

The asymptotic analysis in \citep{shao2010,shao2011} recommends a
minimum exclusion radius\footnote{In terms of angular frequency, this corresponds to the diameter of
$12\pi/T$ in \citet{shao2011}.} of $r=3/T$, which corresponds to assuming a frequency gap of at
least $6/T$ in the finite $T$ setting. Hence, in order to compare
the methods, we also set $r=3/T$ in Algorithm \ref{alg:proc} and assume that $g(T)\geq 6$. The modified threshold $\tau_{\xi}$ corresponding to (\ref{eq:mod_hann_tau}) is then
\begin{equation}
(0.0180+\xi)\sup_{\nu\in[0,B]}|H_c(\nu)|+1.02\min\left\{ \hat{\chi}_{T},\frac{4\alpha\bar{N}_{T}^{1/2}}{(1-\beta)^{1/2}}\left(\frac{\log T}{T}\right)^{1/2}\right\} ,\label{eq:mod_shao_tau}
\end{equation}
where we choose $\xi=0.0001$ to be small. In all our analyses, $\hat{\chi}_{T}$
turns out to be always smaller than the lower bound $4(\bar{N}_{T}\log T/T)^{1/2}$
for the second quantity in the curly bracket. When $g(T)\geq 6$ the
maximum allowable dynamic range widens to 47 under the Hann window.

\subsection{Frequency recovery error rate}
\label{subsec:freqrecovery}
We use the following simulation to empirically study the error rates
for frequency recovery in Proposition \ref{prop:freqrecovery}, which
shows that when $g(T)$ is constant the error $\|\hat{\nu}-\nu^{\lambda}\|_{\infty}$
is no greater than $\mathcal{O}(1/T)$. As remarked after the proposition,
the error rate becomes $\mathcal{O}(T^{-3/2})$ for $g(T)$ equal
to or greater than $\mathcal{O}(T^{1/6})$. In the closely related
time series setting, the unwindowed periodogram achieves the same
rate when $g(T)$ is greater than $\mathcal{O}(T^{1/2})$ (Theorem 6.8b of \citep{li2014}). We will
therefore examine $\|\hat{\nu}-\nu^{\lambda}\|_{\infty}$ as a function
of $T$ at the frequency resolutions corresponding to $g(T)\in\{6,T^{1/6},T^{1/2}\}$. Consider the following class of arrival rates
\begin{equation}
\lambda(t)=7.5+\sum_{k=1}^{5}a_{k}\cos\left(2\pi\left(0.1+(k-1)\frac{g(T)}{T}\right)t+\phi_{k}\right)
\label{eq:convergence-rate}
\end{equation}
whose frequencies are spaced apart by $g(T)/T$. The amplitudes $a_{k}$ are drawn randomly from $U[1,1.5]$,
and the phases $\phi_k$ from $U[0,2\pi)$. For each combination of $T$ and
$g(T)$ we sample 100 sets of the amplitudes and phases, and then
use each set to simulate the corresponding arrival process up to time
$T$.

For the values of $T$ considered, all frequencies are detected by the three methods.
Hence BIC selection and thresholding produce the same results when
applied to the windowed periodogram. Figure \ref{fig:convergence-rateintro}
plots on a log-log scale the error $\|\hat{\nu}-\nu^{\lambda}\|_{\infty}$
averaged across the 100 simulations for each combination of $T$ and
$g(T)$. The slopes of the fitted lines estimate the error rate. For
this example the windowed periodogram performs even better than what
the theory predicts, achieving an error rate of almost $\mathcal{O}(T^{-3/2})$
even for $g(T)=6$. All methods attain this rate when $g(T)=T^{1/2}$.

\subsection{Misspecified arrival rate in the classical resolution setting}

The sawtooth wave in \citep{shao2011} provides a nice example for
testing the robustness of the methods to misspecifications to (\ref{eq:rate_spec}).
Consider the arrival rate
\begin{equation}
\lambda(t)=0.1+0.5\mathrm{mod}(t,2\pi)=0.1+0.5\pi-\sum_{k=1}^{\infty}\frac{\sin\left(2\pi\left(\frac{k}{2\pi}\right)t\right)}{k}\label{eq:sawtooth}
\end{equation}
which has an infinite number of Fourier series frequencies spaced
$1/(2\pi)$ apart. We simulate 100 realizations of the arrival process
up to time $T=1,000$, which is well within the classical setting
where the frequencies are spaced $1/o(T)$ apart. To assess the accuracies
of the three methods at estimating $\lambda(t)$, we use the average
of the MSE $\frac{1}{T}\int_{0}^{T}\{\lambda(t)-\hat{\lambda}(t)\}^{2}dt$
across the 100 samples as the performance metric. We also report the
average number of correct and spurious frequencies\footnote{A correct recovery is defined as one that is within $3/T$ of one of the Fourier frequencies.} recovered by each method in Table \ref{tab:sawtooth}.

Per the discussion in section \ref{sec:findfreq} regarding when the
classic periodogram method should outperform our approach, (\ref{eq:sawtooth}) fits the bill
since the frequencies are separated by much more than order $1/T$.
Interestingly, the differences in performance among the methods are not statistically significant for this example.
\begin{table}[h]
\centering{} %
\begin{tabular}{lrrr}
\toprule
 & UBIC  & WBIC  & WThres \tabularnewline
\midrule
MSE  & 0.17 (0.03)  & 0.19 (0.03)  & 0.17 (0.03) \tabularnewline
\#correct frequencies & 3.05 (0.86)  & 2.77 (0.74)  & 4.41 (1.32) \tabularnewline
\#spurious frequencies & 0.01 (0.10)  & 0.00 (0.00)  & 0.36 (0.77) \tabularnewline
\bottomrule
\end{tabular}\caption{\label{tab:sawtooth}Results for the sawtooth intensity (\ref{eq:sawtooth}).
Column UBIC is the unwindowed periodogram combined with BIC selection,
WBIC is the windowed periodogram with BIC selection, and WThres is
the windowed periodogram with thresholding. Averages over 100 simulations
are reported (standard errors in parentheses).}
\end{table}

\subsection{A super-resolution example with varying dynamic range}

The following arrival rate is inspired by Professor E.H. Kaplan's
analysis of arrivals data to a psychiatric ward, where the existence
of a lunar and a monthly cycle are verified:
\begin{equation}
\lambda(t)=(2r+2)+2r\cos\left(\frac{2\pi}{30}t+2.6\right)+2\cos\left(\frac{2\pi}{28}t+4.5\right).
\label{eq:close-freq}
\end{equation}
The two frequencies at $1/28$ and $1/30$ are separated by a gap
that is slightly larger than $6/T$ when $T=3,000$. The monthly cycle
is $r$ times stronger than the lunar one, meaning that leakage from
the former can easily mask the latter when $r$ is large. The left panels
of Figure \ref{fig:close-freq} display the centralized windowed periodograms
for different values of the dynamic range $r$, and the right panels
display the corresponding unwindowed periodograms. Here we apply thresholding
to the windowed periodograms; BIC selection performs similarly.

For $r=10$ (top row), both periodograms are able to resolve the two
frequencies. For $r=15$ (middle row) only the windowed periodogram
is able to detect the weaker lunar cycle. Both methods fail to identify
the lunar cycle when $r=50$ (bottom row), although the windowed periodogram
is still able to do so for $r=45$ (not shown). This illustrates
the role of windowing in suppressing leakage, thereby allowing for super-resolution frequency recovery. Moreover, our findings match the calculations at the beginning of this section that show the Hann-windowed periodogram
has a maximum allowable dynamic range of 47 when $g(T)\geq6$.
If there are actually more frequencies in (\ref{eq:close-freq}) that are $\mathcal{O}(1/T)$ away from the lunar cycle, then the leakage around $1/28$ in the classic periodogram will be of order $\log p$ as explained in section \ref{sec:overview}. In such cases the classic periodogram may not be able to detect the lunar cycle even if the dynamic range is 1.

\begin{figure}
\noindent \centering{}\includegraphics[scale=0.5]{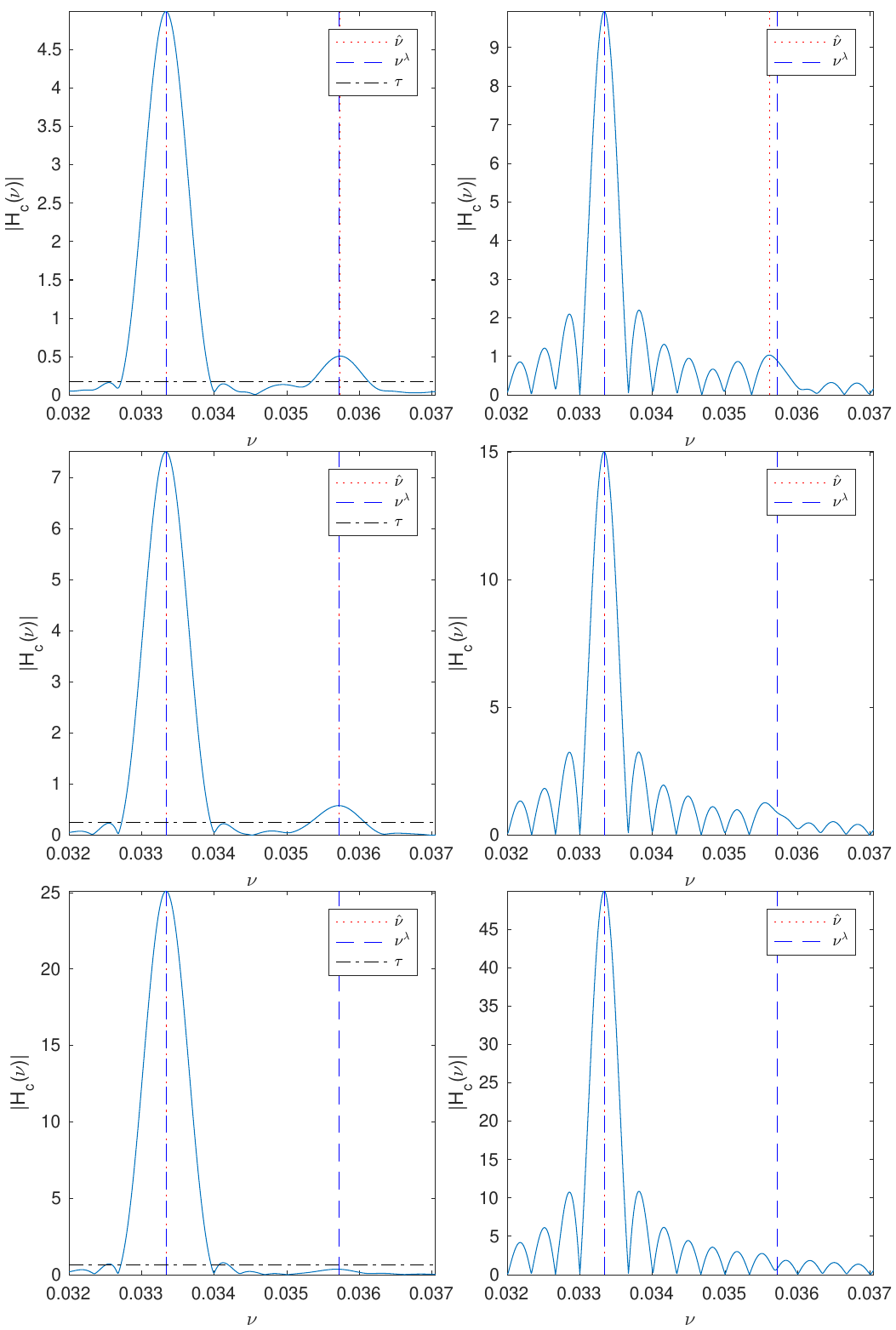}\caption{\label{fig:close-freq}Results for \eqref{eq:close-freq}.\emph{ Left
panels:} The centralized Hann-windowed periodograms. The threshold
is represented by the horizontal line, and the locations of the frequencies
and their estimates are given by the vertical ones. \emph{Right panels:
}The unwindowed centralized periodograms. \textit{Top row}: $r=10$,
\textit{middle row}: $r=15$, \textit{bottom row}: $r=50$.}
\end{figure}

\subsection{Patient arrivals to an emergency department}

Our last example analyzes arrivals data from the emergency department of an academic hospital in the United States. We focus in particular on the arrivals of 66,240 mid-acuity level\footnote{Defined as Emergency Severity Index (ESI) level 2.} patients from 2014 to Q3 of 2015 ($T=652$ days).

As shown in the left panel of Figure \ref{fig:ESI2}, three intraday frequencies and five week-based ones are selected from the centralized periodogram. The intraday frequencies include a daily cycle ($\hat{\nu}_{1}=1.00$),
a 12 hour cycle ($\hat{\nu}_{2}=2.00$), and an 8 hour cycle ($\hat{\nu}_{3}=3.00$). The week-based ones include a weekly cycle ($\hat{\nu}_{4}=0.142$), a half week cycle ($\hat{\nu}_{5}=0.286$), a $1/5$ week cycle ($\hat{\nu}_{6}=0.714$), a $1/6$ week cycle ($\hat{\nu}_{7}=0.857$), and a $1/8$ week cycle ($\hat{\nu}_{8}=1.143$). Given that the fitted rate has a weekly period, we can compare it to the average arrival rate for each of the 168 hours of the week (right panel of Figure \ref{fig:ESI2}). Overall, we see that using 8 frequencies to model the arrival rate does almost as well as using 168 piecewise constant hourly fits. Moreover the sinusoidal estimate reveals two intraday peaks, the first at around 11am and the second at around 5pm. We also see that the intensity of arrivals fade steadily into the weekend.

\begin{figure}
\noindent \centering{}\includegraphics[scale=0.75]{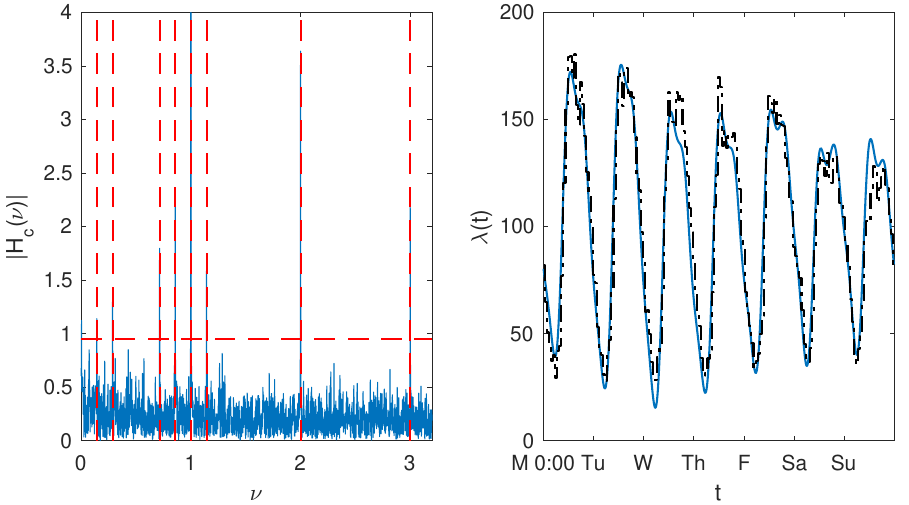}\caption{\label{fig:ESI2}ESI level 2 arrivals. \emph{Left panel: }The centralized windowed periodogram. The selected threshold is represented by the dashed horizontal line, and the location of the frequency estimates	are given by the vertical ones. \emph{Right panel: }The estimated arrival rate (arrivals per day) over the course of a week is given by the solid line. The dash-dot line represents the empirical average	arrival rate for each hour of the week.}
\end{figure}

\section{Discussion\label{sec:discuss}}

By a novel use of windowing, this paper shows that simple periodogram methods can in fact achieve super-resolution frequency recovery for cyclic arrival rates. This improves the resolution of classic periodograms, while being much faster to compute than the SDP approach in super-resolution literature. Under mild assumptions on the dynamic range of the frequency amplitudes, our approach guarantees that no spurious frequencies will be recovered. To establish the consistency of the coefficient estimates, our finite sample results show that if the frequency gap is $1/o(T)$, then the frequencies can be recovered with precision $o(1/T)$ as required. Whether the gap can be relaxed to order $1/T$ is a question that is left for future research.


Another area for future research is to extend the cyclic specification (\ref{eq:rate_spec}) to allow for higher order non-cyclical components as well. One approach is to add wavelets to the basis of complex exponentials. It might then be possible to leverage the rate-optimal procedure in \citet{brown2010} to estimate the time-localized components of the arrival rate.

\section*{Acknowledgements} The review team provided many insightful comments
that significantly improved the paper. Special thanks to Ed Kaplan and Don Green
for stimulating discussions on spectral analysis. The emergency department
arrivals data was kindly provided by Dr. Kito Lord. NC acknowledges the support from the HKUST start-up fund R9382. SNN acknowledges support
from NSF Award DMS 1723128.

\appendix
\section{Proofs}\label{appendix:proofs}

\subsection*{Proofs of (\ref{eq:cmaxUB}) and (\ref{eq:cmaxLB}).}

\begin{proof} Suppose $j\in\arg\max_{k}|c_{k}^{\lambda}|$. Then
it follows from (\ref{eq:periodogram}) that
\[
\begin{aligned}\frac{|\tilde{w}(0)|}{T}\|c^{\lambda}\|_{\infty} & =\left|-H(\nu_{j}^{\lambda})+\frac{1}{T}\sum_{k\neq j}c_{k}^{\lambda}\tilde{w}(\nu_{j}^{\lambda}-\nu_{k}^{\lambda})+\frac{\tilde{\varepsilon}^{w}(\nu_{j}^{\lambda})}{T}\right|\\
 & \leq|H(\nu_{j}^{\lambda})|+\|c^{\lambda}\|_{\infty}\cdot\frac{2}{T}\sum_{l=1}^{\infty}\sup_{|\nu|\geq\frac{4}{T}l}|\tilde{w}(\nu)|+\frac{|\tilde{\varepsilon}^{w}(\nu_{j}^{\lambda})|}{T}\\
 & \leq\sup_{\nu\in[0,B]}|H(\nu)|+S_{2}\|c^{\lambda}\|_{\infty}+\sup_{\nu\in[0,B]}\frac{|\tilde{\varepsilon}^{w}(\nu)|}{T},
\end{aligned}
\]
which establishes (\ref{eq:cmaxUB}). For (\ref{eq:cmaxLB}), let
$\hat{\nu}^{*}=\arg\max_{\nu\in[0,B]}|H(\nu)|$ and suppose $\nu_{k}^{\lambda}$
is the signal frequency closest to $\hat{\nu}^{*}$. If $|\nu_{k}^{\lambda}-\hat{\nu}^{*}|\geq2/T$
then (\ref{eq:leakagefloor}) gives
\begin{equation}
\begin{aligned}|H(\hat{\nu}^{*})| & <\|c^{\lambda}\|_{\infty}\cdot\frac{2}{T}\sum_{l=0}^{\infty}\sup_{|\nu|\geq\frac{2}{T}+\frac{4}{T}l}|\tilde{w}(\nu)|+\sup_{\nu\in[0,B]}\frac{|\tilde{\varepsilon}^{w}(\nu)|}{T}\\
 & =S_{1}\|c^{\lambda}\|_{\infty}+\sup_{\nu\in[0,B]}\frac{|\tilde{\varepsilon}^{w}(\nu)|}{T}.
\end{aligned}
\label{eq:Hbound_altA}
\end{equation}
Consider the alternative $|\nu_{k}^{\lambda}-\hat{\nu}^{*}|<2/T$.
If $\nu_{k}^{\lambda}\leq\hat{\nu}^{*}$ then the $l$-th signal frequency
to the left of $\nu_{k}^{\lambda}$ is at least $4l/T$ away, and
the $l$-th signal frequency to the right of $\hat{\nu}^{*}$ is at
least $\{2+4(l-1)\}/T$ away. Hence

\begin{equation}
\begin{aligned}|H(\hat{\nu}^{*})| & \leq\left(\frac{S_{2}}{2}+\frac{1}{T}\sup_{|\nu|<\frac{2}{T}}|\tilde{w}(\nu)|+\frac{S_{1}}{2}\right)\|c^{\lambda}\|_{\infty}+\sup_{\nu\in[0,B]}\frac{|\tilde{\varepsilon}^{w}(\nu)|}{T}\\
 & =\left\{ \frac{|\tilde{w}(0)|}{T}+\frac{S_{1}+S_{2}}{2}\right\} \|c^{\lambda}\|_{\infty}+\sup_{\nu\in[0,B]}\frac{|\tilde{\varepsilon}^{w}(\nu)|}{T}.
\end{aligned}
\label{eq:Hbound_altB}
\end{equation}
The same bound also applies when $\nu_{k}^{\lambda}>\hat{\nu}^{*}$,
so combining (\ref{eq:Hbound_altA}) and (\ref{eq:Hbound_altB}) gives
(\ref{eq:cmaxLB}).\end{proof}

\subsection*{Proof of Lemma \ref{lem:hann_tailsum}}

We begin by estimating the side lobe heights of the Hann window's
spectrum.

\begin{lem} \label{lem:sidelobepeaks}For $\tilde{w}(\nu)$ defined
in (\ref{eq:hann}), let $\nu_{k}^{w}$ be the location of the peak
of $|\tilde{w}(\nu)|$'s side lobe in the interval $(\frac{k}{T},\frac{k+1}{T})$
for $k\geq2$. Then $\nu_{k}^{wL}<\nu_{k}^{w}<\frac{k+1/2}{T}$ where
$\nu_{k}^{wL}$ is the larger root of $(k+1/2-T\nu)(T\nu-1/k)=3/\pi^{2}$.
Hence
\[
|\tilde{w}(\nu_{k}^{w})|>\left|\tilde{w}\left(\frac{k+1/2}{T}\right)\right|=\frac{T/(2\pi)}{\left(k+\frac{1}{2}\right)\left\{ \left(k+\frac{1}{2}\right)^{2}-1\right\} }\geq\frac{32T}{105\pi k^{3}},
\]
\[
|\tilde{w}(\nu_{k}^{w})|<\frac{T/(2\pi)}{T\nu_{k}^{w}\left\{ (T\nu_{k}^{w})^{2}-1\right\} }<\frac{T/(2\pi)}{T\nu_{k}^{wL}\left\{ (T\nu_{k}^{wL})^{2}-1\right\} }<\frac{T}{2\pi k^{3}}.
\]
\end{lem}

\begin{proof} By symmetry it suffices to consider the heights of
the side lobes of $(T/2)\sinc(T\nu)/\{(T\nu)^{2}-1\}$ over intervals
$(\frac{k}{T},\frac{k+1}{T})$ for $k$ even, and $(T/2)\sinc(T\nu)/\{1-(T\nu)^{2}\}$
for odd $k\geq3$. The first order condition implies that $\nu_{k}^{w}$
is the root of
\[
\cot\pi T\nu=\frac{3(T\nu)^{2}-1}{\pi T\nu((T\nu)^{2}-1)}.
\]
Over $(\frac{k}{T},\frac{k+1}{T})$ the left hand side is decreasing
from $+\infty$ to $-\infty$ and crosses zero at $\nu=\frac{k+1/2}{T}$.
The right hand side is positive and also decreasing, therefore the
two sides intersect somewhere in $(\frac{k}{T},\frac{k+1/2}{T})$.
On this subinterval, linearizing $\cot\pi T\nu$ about $\nu=\frac{k+1/2}{T}$
yields the lower bound $\pi(k+1/2-T\nu)$. The locations at which
this intersects
\[
\frac{3(T\nu)^{2}-1}{\pi T\nu((T\nu)^{2}-1)}<\frac{3(T\nu)^{2}}{\pi T\nu((T\nu)^{2}-1)}\leq\frac{3}{\pi\left(T\nu-1/k\right)}
\]
must all be less than $\nu_{k}^{w}$. Therefore the larger root of
$\pi(k+1/2-T\nu)=3/\{\pi(T\nu-1/k)\}$,
\[
\nu_{k}^{wL}=\frac{1}{2}\left(\frac{k+1/2}{T}+\frac{1}{kT}\right)+\left\{ \frac{1}{4}\left(\frac{k+1/2}{T}-\frac{1}{kT}\right)^{2}-\frac{3}{T^{2}\pi^{2}}\right\} ^{1/2},
\]
is a lower bound for $\nu_{k}^{w}$. The bounds on $|\tilde{w}(\nu_{k}^{w})|$
in the lemma follow directly from the bounds on $\nu_{k}^{w}$, and
from noting that $T\nu_{k}^{wL}\left\{ (T\nu_{k}^{wL})^{2}-1\right\} >k^{3}$
and $(k+1/2)\left\{ (k+1/2)^{2}-1\right\} \leq105k^{3}/64$ for $k\geq2$.
\end{proof}

The bounds for $|\tilde{w}(\nu_{k}^{w})|$ allow us to prove Lemma
\ref{lem:hann_tailsum}.

\begin{proof} Since $|\tilde{w}(\nu_{k}^{w})|=\max_{\nu\in(\frac{k}{T},\frac{k+1}{T})}|\tilde{w}(\nu)|$
is decreasing in $k\geq2$, $\sup_{|\nu|\geq k/T}|\tilde{w}(\nu)|=|\tilde{w}(\nu_{k}^{w})|$.
Furthermore it can be verified that $|\tilde{w}(\nu_{2}^{w})|+|\tilde{w}(\nu_{6}^{w})|$
is numerically between $0.013954T$ and $0.013955T$. Hence for $S_{1}$,
\[
\begin{aligned}\frac{2}{T}\sum_{l=0}^{\infty}\sup_{|\nu|\geq\frac{2}{T}+\frac{4}{T}l}|\tilde{w}(\nu)| & <\frac{2}{T}\left\{ 0.013955T+\sum_{l=2}^{100}\frac{T/(2\pi)}{T\nu_{2+4l}^{wL}\left\{ (T\nu_{2+4l}^{wL})^{2}-1\right\} }+\sum_{l=101}^{\infty}\frac{T/(2\pi)}{(2+4l)^{3}}\right\} \\
 & =2\left\{ 0.013955+\sum_{l=2}^{100}\frac{1/(2\pi)}{T\nu_{2+4l}^{wL}\left\{ (T\nu_{2+4l}^{wL})^{2}-1\right\} }-\frac{\psi^{(2)}(203/2)}{256\pi}\right\} \\
 & <0.02844,
\end{aligned}
\]
where the tail is bounded using the polygamma function $\psi^{(2)}(z)$
of order 2. The corresponding lower estimate is
\[
\begin{aligned}\frac{2}{T}\sum_{l=0}^{\infty}\sup_{|\nu|\geq\frac{2}{T}+\frac{4}{T}l}|\tilde{w}(\nu)| & >\frac{2}{T}\left\{ 0.013954T+\sum_{l=2}^{100}\frac{T/\{2\pi(2+4l+\frac{1}{2})\}}{(2+4l+\frac{1}{2})^{2}-1}\right\} \\
 & >0.02843.
\end{aligned}
\]
The tail sum $S_{2}$ can be estimated in the same manner. To derive
the remaining bounds stated in the lemma, define
\[
W(\nu)=\frac{1/(2\pi)}{|T\nu|\{(T\nu)^{2}-1\}}
\]
so that $|\tilde{w}(\nu)|/T\leq W(\nu)$ for $|\nu|>1/T$. If $\nu\in(\nu_{k}^{\lambda}-\frac{2}{T},\nu_{k}^{\lambda}+\frac{2}{T})$
then $|\nu-\nu_{l}^{\lambda}|>m_{k,l}\frac{g(T)}{T}-\frac{2}{T}$
for some integer $m_{k,l}\geq1$ under A1. Hence
\[
\frac{1}{T}\sum_{l\neq k}|\tilde{w}(\nu-\nu_{l}^{\lambda})|<2\sum_{m=1}^{\infty}W\left(\frac{mg(T)-2}{T}\right).
\]
Bearing in mind that $g(T)\geq4$,
\[
\begin{aligned}(mg(T)-2)\{(mg(T)-2)^{2}-1\} & =g(T)^{3}\left(m-\frac{1}{g(T)}\right)\left(m-\frac{2}{g(T)}\right)\left(m-\frac{3}{g(T)}\right)\\
 & \geq g(T)^{3}\left(m-\frac{1}{4}\right)\left(m-\frac{2}{4}\right)\left(m-\frac{3}{4}\right).
\end{aligned}
\]
Therefore
\[
\frac{1}{T}\sum_{l\neq k}|\tilde{w}(\nu-\nu_{l}^{\lambda})|<\frac{1/\pi}{g(T)^{3}}\sum_{m=1}^{\infty}\frac{1}{\left(m-\frac{1}{4}\right)\left(m-\frac{2}{4}\right)\left(m-\frac{3}{4}\right)}=\frac{16\log2}{\pi g(T)^{3}}<\frac{4}{g(T)^{3}}.
\]
The bound on the sum of the derivatives follows from some algebra showing that $|\tilde{w}'(\nu-\nu^\lambda_l)|/T < (2\pi+11/6)T\cdot W(\nu-\nu^\lambda_l)$.
\end{proof}

\subsection*{Proof of Lemma \ref{lem:dualnorm}}

Lemma \ref{lem:subexpconc} below is required for the proof of Lemma
\ref{lem:dualnorm}. It gives a concentration bound for weighted sums
of Poisson process increments using standard results for sub-exponential
variables.

\begin{lem} \label{lem:subexpconc}Suppose $\{Z_{j}\}_{j=0}^{L-1}$
are independent and centred Poisson random variables with rates $\left\{ \Lambda\left(\frac{(j+1)T}{L}\right)-\Lambda\left(\frac{jT}{L}\right)\right\} _{j}$,
and let the constants $a_{j}$ satisfy $\max_{j}|a_{j}|\leq1$. Then
for $z>0$,
\[
\mathbb{P}\left(\sum_{j=0}^{L-1}a_{j}Z_{j}\geq z\right)\leq\exp\left\{ -\min\left(\frac{z^{2}}{4\Lambda(T)},\frac{z}{2}\right)\right\}
\]
and
\[
\mathbb{P}\left(\left|\sum_{j=0}^{L-1}a_{j}Z_{j}\right|\geq z\right)\leq2\exp\left\{ -\min\left(\frac{z^{2}}{4\Lambda(T)},\frac{z}{2}\right)\right\} .
\]
\end{lem}

\begin{proof} Recall that for a centred Poisson random variable $Z$
with rate $\mu$
\[
\mathbb{E}\exp(sZ)=\exp\{\mu(e^{s}-1-s)\},
\]
and that for $s\leq1$ we have
\[
\exp(\mu(e^{s}-1-s))\leq\exp(\mu s^{2}).
\]
The claim is clear for $s\leq0$. For $s\in(0,1]$ one can show that
$e^{s}-1-s-s^{2}/2\leq s^{2}/2$ by comparing its power series to
a dominating geometric sum. Furthermore note that $sa_{j}\leq1$ for
any $s\in(0,1]$, hence optimizing Chernoff's bound for $\sum_{j}a_{j}Z_{j}$
within this range gives
\[
\begin{aligned}\mathbb{P}\left(\sum_{j=0}^{L-1}a_{j}Z_{j}\geq z\right) & \leq\min_{0<s\leq1}\exp\left(s^{2}\sum_{j=0}^{L-1}\left\{ \Lambda\left(\frac{(j+1)T}{L}\right)-\Lambda\left(\frac{jT}{L}\right)\right\} a_{j}^{2}-sz\right)\\
 & \leq\min_{0<s\leq1}\exp\left(s^{2}\sum_{j=0}^{L-1}\left\{ \Lambda\left(\frac{(j+1)T}{L}\right)-\Lambda\left(\frac{jT}{L}\right)\right\} -sz\right)\\
 & \leq\min_{0<s\leq1}\exp\left(s^{2}\Lambda(T)-sz\right),
\end{aligned}
\]
with the minimizer being
\[
s^{*}=\begin{cases}
\frac{z}{2\Lambda(T)} & z\leq2\Lambda(T)\\
1 & \Lambda(T)-z<-z/2
\end{cases}.
\]
\end{proof}

We now prove Lemma \ref{lem:dualnorm}.

\begin{proof} Recall from (\ref{eq:doobmeyer}) and (\ref{eq:periodogram})
that $\tilde{\varepsilon}^{w}(\nu)=\int_{0}^{T}e^{-2\pi i\nu t}w(t)d\{N(t)-\Lambda(t)\}$.
Partitioning $[0,T]$ into $L$ intervals each of width $\Delta=T/L$,
we find that for $\nu\in[0,B]$,
\begin{equation}
\begin{aligned}|\tilde{\varepsilon}^{w}(\nu)| & =\left|\sum_{j=0}^{L-1}\int_{j\Delta}^{(j+1)\Delta}e^{-2\pi i\nu t}w(t)d\{N(t)-\Lambda(t)\}\right|\\
 & =\left|\sum_{j=0}^{L-1}e^{-2\pi i\nu j\Delta}\int_{0}^{\Delta}e^{-2\pi i\nu t}w(t+j\Delta)d\{N(t+j\Delta)-\Lambda(t+j\Delta)\}\right|\\
 & \leq\sum_{j=0}^{L-1}\int_{0}^{\Delta}|e^{-2\pi i\nu t}w(t+j\Delta)-w(j\Delta)|d\{N(t+j\Delta)+\Lambda(t+j\Delta)\}\\
 & +\left|\sum_{j=0}^{L-1}e^{-2\pi i\nu j\Delta}w(j\Delta)[N((j+1)\Delta)-N(j\Delta)-\{\Lambda((j+1)\Delta)-\Lambda(j\Delta)\}]\right|\\
 & \leq\Delta\left(2\pi B+\sup_{t\in(0,T)}|w'(t)|\right)\{N(T)+\Lambda(T)\}+\left|\sum_{j=0}^{L-1}e^{-2\pi i\nu j\Delta}w(j\Delta)Z_{j}\right|
\end{aligned}
\label{eq:poissonFT}
\end{equation}
where $\{Z_{j}\}$ are independent and centred Poisson random variables
with rates $\mu_{j}=\Lambda((j+1)\Delta)-\Lambda(j\Delta)$. By rewriting
$N(T)+\Lambda(T)=2\Lambda(T)+\{N(T)-\Lambda(T)\}$, it follows from
Lemma \ref{lem:subexpconc} in Appendix \ref{appendix:proofs} that the first term above
exceeds
\begin{equation}
\Delta\left(2\pi B+\sup_{t\in(0,T)}|w'(t)|\right)(2\Lambda(T)+z_{1})\label{eq:biginc}
\end{equation}
with probability less than $\exp\left\{ -\min\left(\frac{z_{1}^{2}}{4\Lambda(T)},\frac{z_{1}}{2}\right)\right\} $.
To control the supremum over $\nu\in[0,B]$ of the last term in (\ref{eq:poissonFT}),
we express it in the dual norm terminology of \citep{bhaskar}:
\[
\|wZ\|_{\mathcal{A}}^{*}=\sup_{f\in[0,B\Delta]}\left|\sum_{j=0}^{L-1}w(j\Delta)Z_{j}e^{-2\pi ijf}\right|.
\]
According to Appendix C of \citep{bhaskar} the dual norm and its
approximation
\[
\|wZ\|_{\mathcal{A}_{K}}^{*}=\max_{f\in\left\{ 0,\frac{B\Delta}{K},\dots,\frac{(K-1)B\Delta}{K}\right\} }\left|\sum_{j=0}^{L-1}w(j\Delta)Z_{j}e^{-2\pi ijf}\right|
\]
are equivalent:
\[
\|wZ\|_{\mathcal{A}_{K}}^{*}\le\|wZ\|_{\mathcal{A}}^{*}\le\left(1-\frac{2\pi LB\Delta}{K}\right)^{-1}\|wZ\|_{\mathcal{A}_{K}}^{*}=\left(1-\frac{2\pi BT}{K}\right)^{-1}\|wZ\|_{\mathcal{A}_{K}}^{*}.
\]
Therefore, to bound $\|wZ\|_{\mathcal{A}}^{*}$, note that
\[
\begin{aligned}\|wZ\|_{\mathcal{A}_{K}}^{*} & \leq\max_{f\in\left\{ 0,\frac{B\Delta}{K},\dots,\frac{(K-1)B\Delta}{K}\right\} }\left(\left|\sum_{j=0}^{L-1}w(j\Delta)\cos(2\pi jf)Z_{j}\right|+\left|\sum_{j=0}^{L-1}w(j\Delta)\sin(2\pi jf)Z_{j}\right|\right)\\
 & =\max_{f\in\left\{ 0,\frac{B\Delta}{K},\dots,\frac{(K-1)B\Delta}{K}\right\} }(|U_{f}|+|V_{f}|).
\end{aligned}
\]
By increasing $B$ if necessary so that $K=2\gamma\pi BT$ is an integer,
\begin{equation}
\begin{aligned}\mathbb{P}\left(\|wZ\|_{\mathcal{A}}^{*}\geq z_{2}\right) & \leq\mathbb{P}\left(\|wZ\|_{\mathcal{A}_{K}}^{*}\geq\frac{\gamma-1}{\gamma}z_{2}\right)\\
 & \leq\mathbb{P}\left(\max_{f\in\left\{ 0,\frac{B\Delta}{K},\dots,\frac{(K-1)B\Delta}{K}\right\} }(|U_{f}|+|V_{f}|)\geq\frac{\gamma-1}{\gamma}z_{2}\right)\\
 & \leq K\max_{f\in\left\{ 0,\frac{B\Delta}{K},\dots,\frac{(K-1)B\Delta}{K}\right\} }\mathbb{P}\left(|U_{f}|+|V_{f}|\geq\frac{\gamma-1}{\gamma}z_{2}\right)\\
 & \leq K\max_{f\in\left\{ 0,\frac{B\Delta}{K},\dots,\frac{(K-1)B\Delta}{K}\right\} }\left\{ \mathbb{P}\left(|U_{f}|\geq\frac{\gamma-1}{2\gamma}z_{2}\right)+\mathbb{P}\left(|V_{f}|\geq\frac{\gamma-1}{2\gamma}z_{2}\right)\right\} \\
 & \leq2\gamma\pi BT\cdot4\exp\left[-\min\left\{ \left(\frac{\gamma-1}{4\gamma\Lambda(T)^{1/2}}z_{2}\right)^{2},\frac{\gamma-1}{4\gamma}z_{2}\right\} \right],
\end{aligned}
\label{eq:smallinc}
\end{equation}
where the third inequality follows from the union bound and the fourth
one from $\{|U_{f}|+|V_{f}|\geq z_{2}/2\}\subseteq\{|U_{f}|\geq z_{2}/4\}\cup\{|V_{f}|\geq z_{2}/4\}$.
The last inequality follows from Lemma \ref{lem:subexpconc}.

Substituting the bounds (\ref{eq:biginc}) and (\ref{eq:smallinc})
into (\ref{eq:poissonFT}) reveals that
\[
\sup_{\nu\in[0,B]}|\tilde{\varepsilon}^{w}(\nu)|<\Delta\left(2\pi B+\sup_{t\in(0,T)}|w'(t)|\right)(2\Lambda(T)+z_{1})+z_{2}
\]
with probability at least
\[
1-\exp\left\{ -\min\left(\frac{z_{1}^{2}}{4\Lambda(T)},\frac{z_{1}}{2}\right)\right\} -8\gamma\pi BT\exp\left[-\min\left\{ \left(\frac{\gamma-1}{4\gamma\Lambda(T)^{1/2}}z_{2}\right)^{2},\frac{\gamma-1}{4\gamma}z_{2}\right\} \right].
\]
Since this holds for arbitrary $\Delta>0$, setting $z_{2}=4\alpha(\Lambda(T)\log T)^{1/2}$
gives
\[
\sup_{\nu\in[0,B]}\frac{|\tilde{\varepsilon}^{w}(\nu)|}{T}<4\alpha\left(\bar{\Lambda}_{T}\frac{\log T}{T}\right)^{1/2}
\]
with probability at least
\[
1-8\gamma\pi B\left[1/T^{\left(\frac{\gamma-1}{\gamma}\alpha\right)^{2}-1} + T\exp\left\{ -(\Lambda(T)\log T)^{1/2}\right\} \right].
\]
Finally, it is easy to show using Chernoff's bound that
\[
\mathbb{P}\left(1-\beta<\frac{N(T)}{\Lambda(T)}<1+\beta\right)\geq1-e^{-\Lambda(T)\beta^{2}/2}-e^{-\Lambda(T)\beta^{2}/4}\geq1-2e^{-\Lambda(T)\beta^{2}/4}.
\]
\end{proof}

\subsection*{Proof of Proposition \ref{prop:freqrecovery}}
The proof requires the use of two results. First, it is easy to show
for the Hann window that
\begin{equation}
\frac{1}{2}\left\{ 1-\frac{3}{2}(T\nu)^{2}\right\} \leq\frac{|\tilde{w}(\nu)|}{T}\leq\frac{1}{2}\left\{ 1-\frac{1}{4}(T\nu)^{2}\right\} .\label{eq:w_quadbounds}
\end{equation}
Second, under the setting of Lemma \ref{lem:dualnorm},
\begin{equation}
\sup_{\nu\in[0,B]}\frac{|(\tilde{\varepsilon}^{w})'(\nu)|}{T}<8\pi\alpha\bar{\Lambda}_{T}^{1/2}(T\log T)^{1/2}.\label{eq:dualderivative}
\end{equation}
To see this, note that $|(\tilde{\varepsilon}^{w})'(\nu)|=2\pi T\left|\int_{0}^{T}e^{-2\pi i\nu t}\{tw(t)/T\}d\{N(t)-\Lambda(t)\}\right|$.
The window $v(t)=tw(t)/T$ satisfies $\sup_{t\in[0,T]}v(t)\leq1$
and $\sup_{t\in[0,T]}v'(t)<\infty$ under the conditions in the lemma.
Hence (\ref{eq:dualderivative}) follows from applying Lemma \ref{lem:dualnorm}.

\begin{proof}[Proof of proposition] That $\|\nu^{\lambda}-\hat{\nu}\|_{\infty}<2/T$
with high probability is clear from the development of section \ref{sec:findfreq}.
We will use this as the starting point for obtaining a sharper bound
in the region $\epsilon(T)\leq 87/40$.

By a unitary renormalization of $H(\nu)$ and a time shift, we may assume that $c_{k}^{\lambda}$ is
real and positive and that $\tilde{w}(\nu)=\frac{T}{2}\frac{\sinc(T\nu)}{1-(T\nu)^{2}}$. Rewrite (\ref{eq:periodogram})
as
\[
\begin{aligned}H(\nu) & =\frac{c_{k}^{\lambda}}{T}\tilde{w}(\nu-\nu_{k}^{\lambda})+\frac{1}{T}\sum_{l\neq k}c_{l}^{\lambda}\tilde{w}(\nu-\nu_{l}^{\lambda})+\frac{\tilde{\varepsilon}^{w}(\nu)}{T}\\
 & =\frac{c_{k}^{\lambda}}{T}\tilde{w}(\nu-\nu_{k}^{\lambda})+\eta_{k}(\nu).
\end{aligned}
\]
By the local optimality of $\hat{\nu}_{k}$,
\begin{align}
0 & \leq|H(\hat{\nu}_{k})|-|H(\nu_{k}^{\lambda})|\nonumber \\
 & =\left|c_{k}^{\lambda}\frac{\tilde{w}(\hat{\nu}_{k}-\nu_{k}^{\lambda})}{T}+\eta_{k}(\hat{\nu}_{k})\right|-\left|\frac{c_{k}^{\lambda}}{2}+\eta_{k}(\nu_{k}^{\lambda})\right|,\label{eq:manubasic}
\end{align}
so combining this with the right hand side of (\ref{eq:w_quadbounds})
yields
\begin{equation}
\frac{c_{k}^{\lambda}}{8}\{T(\hat{\nu}_{k}-\nu_{k}^{\lambda})\}^{2}\leq\frac{c_{k}^{\lambda}}{2}-\left|c_{k}^{\lambda}\frac{\tilde{w}(\hat{\nu}_{k}-\nu_{k}^{\lambda})}{T}\right|\leq2\sup_{\nu:|\nu-\nu_{k}^{\lambda}|<\frac{2}{T}}|\eta_{k}(\nu)|.\label{eq:nubound1}
\end{equation}
Returning to \eqref{eq:manubasic}, an application of the mean value
theorem to $\eta_{k}(\hat{\nu}_{k})$ yields
\begin{align*}
0 & \leq\left|c_{k}^{\lambda}\frac{\tilde{w}(\hat{\nu}_{k}-\nu_{k}^{\lambda})}{T}+\eta_{k}(\nu_{k}^{\lambda})\right|+|\eta_{k}'(s_{1})|\cdot|\hat{\nu}_{k}-\nu_{k}^{\lambda}|-\left|\frac{c_{k}^{\lambda}}{2}+\eta_{k}(\nu_{k}^{\lambda})\right|
\end{align*}
for some $s_{1}$ between $\hat{\nu}_{k}$ and $\nu_{k}^{\lambda}$.
Since $\tilde{w}(\hat{\nu}_{k}-\nu_{k}^{\lambda})$ is assumed to
be real-valued, another application of the mean value theorem to $x\mapsto|x+iy|$
for $y$ fixed gives
\begin{equation}
\begin{aligned}0 & \leq\left|c_{k}^{\lambda}\frac{\tilde{w}(\hat{\nu}_{k}-\nu_{k}^{\lambda})}{T}-\frac{c_{k}^{\lambda}}{2}+\frac{c_{k}^{\lambda}}{2}+\eta_{k}(\nu_{k}^{\lambda})\right|+|\eta_{k}'(s_{1})|\cdot|\hat{\nu}_{k}-\nu_{k}^{\lambda}|-\left|\frac{c_{k}^{\lambda}}{2}+\eta_{k}(\nu_{k}^{\lambda})\right|\\
 & =\frac{\realp\left\{ s_{2}+\eta_{k}(\nu_{k}^{\lambda})\right\} }{\left|s_{2}+\eta_{k}(\nu_{k}^{\lambda})\right|}\left(\frac{\tilde{w}(\hat{\nu}_{k}-\nu_{k}^{\lambda})}{T}-\frac{1}{2}\right)c_{k}^{\lambda}+|\eta_{k}'(s_{1})|\cdot|\hat{\nu}_{k}-\nu_{k}^{\lambda}|
\end{aligned}
\label{eq:manutwo}
\end{equation}
for some $\frac{c_{k}^{\lambda}}{T}\tilde{w}(\hat{\nu}_{k}-\nu_{k}^{\lambda})<s_{2}<\frac{c_{k}^{\lambda}}{2}$.
In view of the left hand side of (\ref{eq:w_quadbounds}) as well as (\ref{eq:nubound1}), $s_{2}>\frac{c_{k}^{\lambda}}{2}-12\sup_{\nu:|\nu-\nu_{k}^{\lambda}|<\frac{2}{T}}|\eta_{k}(\nu)|$.
Furthermore, Lemmas \ref{lem:hann_tailsum} and \ref{lem:dualnorm} imply that
\[
\sup_{\nu:|\nu-\nu_{k}^{\lambda}|<\frac{2}{T}}|\eta_{k}(\nu)| < 4\left(\|c^{\lambda}\|_{\infty}+\alpha\bar{\Lambda}_{T}^{1/2}\right)\max\left\{ \frac{1}{g(T)^{3}},\left(\frac{\log T}{T}\right)^{1/2}\right\} \leq\frac{c_{k}^{\lambda}}{40},
\]
where the last inequality comes from $\epsilon(T)\leq87/40$. Since
$\frac{1}{T}\tilde{w}(\hat{\nu}_{k}-\nu_{k}^{\lambda})-\frac{1}{2}\leq0$, we desire a lower bound for $\frac{ \realp\{s_2+\eta_k(\nu^\lambda_k)\} }{ |s_2+\eta_k(\nu^\lambda_k)| }$ in (\ref{eq:manutwo}). Putting the bounds for $s_{2}$ and $|\eta_{k}(\nu)|$ into \eqref{eq:manutwo} yields
\[
\begin{aligned}0 & < \frac{\frac{c_{k}^{\lambda}}{2}-11\sup_\nu |\eta_{k}(\nu)|}
{\frac{c_{k}^{\lambda}}{2}+\sup_\nu |\eta_{k}(\nu)|}\left(\frac{\tilde{w}(\hat{\nu}_{k}-\nu_{k}^{\lambda})}{T}-\frac{1}{2}\right)c_{k}^{\lambda}+|\eta_{k}'(s_{1})|\cdot|\hat{\nu}_{k}-\nu_{k}^{\lambda}|\\
 & \leq\frac{c_{k}^{\lambda}}{3}\left(\frac{\tilde{w}(\hat{\nu}_{k}-\nu_{k}^{\lambda})}{T}-\frac{1}{2}\right)+|\eta_{k}'(s_{1})|\cdot|\hat{\nu}_{k}-\nu_{k}^{\lambda}|,
\end{aligned}
\]
which when combined with the right hand side of (\ref{eq:w_quadbounds})
gives $|\hat{\nu}_{k}-\nu_{k}^{\lambda}| < \frac{24}{c_{k}^{\lambda}T^{2}}|\eta_{k}'(s_{1})|$.
The derivative can be bounded using Lemma \ref{lem:hann_tailsum}
and (\ref{eq:dualderivative}):
\[
\begin{aligned}|\hat{\nu}_{k}-\nu_{k}^{\lambda}| & < \frac{24}{c_{k}^{\lambda}T^{2}}\left\{ \|c^{\lambda}\|_{\infty}\frac{29T}{g(T)^{3}}+8\pi\alpha\bar{\Lambda}_{T}^{1/2}(T\log T)^{1/2}\right\} \\
 & <\frac{2}{T}\cdot348 \frac{\|c^{\lambda}\|_{\infty}+\alpha\bar{\Lambda}_{T}^{1/2}}{\min_{k}|c_{k}^{\lambda}|} \max\left\{ \frac{1}{g(T)^{3}},\left(\frac{\log T}{T}\right)^{1/2}\right\} .
\end{aligned}
\]
\end{proof}

\subsection*{Proof of Corollary \ref{cor:mod_tau}}
\begin{proof} Specializing (\ref{eq:cmaxUB}) and (\ref{eq:cmaxLB})
to the Hann window gives
\[
0.49535\|c^{\lambda}\|_{\infty}-\frac{\sup_{\nu}|\tilde{\varepsilon}^{w}(\nu)|}{T}\leq\sup_{\nu}|H(\nu)|\leq0.51655\|c^{\lambda}\|_{\infty}+\frac{\sup_{\nu}|\tilde{\varepsilon}^{w}(\nu)|}{T},
\]
which will used throughout to bound $\sup_{\nu\in[0,B]}|H(\nu)|$
in terms of $\|c^{\lambda}\|_{\infty}$ and vice versa. In addition,
Lemma \ref{lem:dualnorm} shows that with the stated probability the
spectral noise level is controlled by
\[
\begin{aligned}\sup_{\nu}\frac{|\tilde{\varepsilon}^{w}(\nu)|}{T} & <\frac{4\alpha\bar{N}_{T}^{1/2}}{(1-\beta)^{1/2}}\left(\frac{\log T}{T}\right)^{1/2}<4\alpha\left(\frac{1+\beta}{1-\beta}\cdot\frac{\bar{\Lambda}_{T}\log T}{T}\right)^{1/2}\\
 & \leq\frac{120\xi}{28+25\xi}\|c^{\lambda}\|_{\infty}<\frac{0.48345\xi}{1.1174+\xi}\|c^{\lambda}\|_{\infty}\\
 & <\frac{50\xi}{53+50\xi}\times0.49535\|c^{\lambda}\|_{\infty}\\
 & <\xi\|c^{\lambda}\|_{\infty}.
\end{aligned}
\]

When $\nu$ is at least $2/T$ away from the nearest $\nu_{k}^{\lambda}$,
(\ref{eq:tau}) tells us that
\[
\begin{aligned}|H(\nu)| & <0.0574\sup_{\nu}|H(\nu)|+1.06\frac{\sup_{\nu}|\tilde{\varepsilon}^{w}(\nu)|}{T}\\
 & <0.0574\sup_{\nu}|H(\nu)|+1.06\cdot\frac{50\xi}{53+50\xi}\times0.49535\|c^{\lambda}\|_{\infty}\\
 & \leq0.0574\sup_{\nu}|H(\nu)|+1.06\cdot\frac{50\xi}{53+50\xi}\left(\sup_{\nu}|H(\nu)|+\frac{\sup_{\nu}|\tilde{\varepsilon}^{w}(\nu)|}{T}\right)\\
 & <0.0574\sup_{\nu}|H(\nu)|+\left\{ 1.06\sum_{m=1}^{\infty}\left(\frac{50\xi}{53+50\xi}\right)^{m}\right\} \sup_{\nu}|H(\nu)|\\
 & =(0.0574+\xi)\sup_{\nu}|H(\nu)|<\tau_{\xi}.
\end{aligned}
\]
Hence, no spurious frequencies will be selected. To select the $k$-th
frequency it suffices for $|H(\nu_{k}^{\lambda})|>\tau_{\xi}$. Along
the lines of deriving (\ref{eq:coefsizes}) we see that
\[
\begin{aligned}|H(\nu_{k}^{\lambda})| & \geq\frac{1}{2}|c_{k}^{\lambda}|-0.00465\|c^{\lambda}\|_{\infty}-\frac{\sup_{\nu}|\tilde{\varepsilon}^{w}(\nu)|}{T}\\
 & >\frac{1}{2}|c_{k}^{\lambda}|-(0.00465+\xi)\|c^{\lambda}\|_{\infty}.
\end{aligned}
\]
Furthermore
\[
\begin{aligned}\tau_{\xi} & \leq(0.0574+\xi)\times0.51655\|c^{\lambda}\|_{\infty}+(0.0574+\xi+1.06)\frac{4\alpha\bar{N}_{T}^{1/2}}{(1-\beta)^{1/2}}\left(\frac{\log T}{T}\right)^{1/2}\\
 & <(0.02965+0.51655\xi)\|c^{\lambda}\|_{\infty}+0.48345\xi\|c^{\lambda}\|_{\infty}\\
 & =(0.02965+\xi)\|c^{\lambda}\|_{\infty},
\end{aligned}
\]
so it follows from the strengthened version of A2 that $|H(\nu_{k}^{\lambda})|>\tau_{\xi}$
for all $k$, and we inherit the estimation precision of Proposition \ref{prop:freqrecovery}.
\end{proof}

\subsection*{Proof of Proposition \ref{prop:LSerror}}

The following bound is needed in the proof below: For any frequency
pair $|\nu-\omega|<\kappa\epsilon(T)/T$,
\begin{equation}
\frac{1}{T}\left|\tilde{I}_{(0,T]}(\omega)-\tilde{I}_{(0,T]}(\nu)\right|< \kappa\pi\epsilon(T).\label{eq:dhann}
\end{equation}
This follows from
\[
\begin{aligned}\frac{1}{T}\left|\tilde{I}_{(0,T]}(\omega)-\tilde{I}_{(0,T]}(\nu)\right| & =\frac{1}{T}\left|\int_{0}^{T}e^{-2\pi i\omega t}dt-\int_{0}^{T}e^{-2\pi i\nu t}dt\right|\\
 & \leq\frac{1}{T}\int_{0}^{T}\left|e^{-2\pi i(\omega-\nu)t}-1\right|dt\\
 & =\frac{2}{T}\int_{0}^{T}|\sin(\omega-\nu)\pi t|dt\\
 & \leq\pi T|\omega-\nu|<\kappa\pi\epsilon(T),
\end{aligned}
\]
where the inequality obtained from interchanging the modulus and integral
is valid for complex-valued integrals. The second equality follows
from $|e^{-2\pi ix}-1|^{2}=4\sin^{2}\pi x$, and the penultimate inequality
from $|\sin x|\leq|x|$.

We will also need a matrix norm for the proof: When the collection
of $(p+1)\times(p+1)$ complex-valued matrices is equipped with the
maximum row sum norm
\[
\|A\|=\max_{j}\sum_{k}|A_{jk}|\Longrightarrow\|Ac\|_{\infty}\leq\|A\|\|c\|_{\infty},
\]
it becomes a Banach algebra because $\|\cdot\|$ is submultiplicative.
Hence the resolvent $(I+A)^{-1}$ admits the expansion $\sum_{m=0}^{\infty}(-A)^{m}$
for $\|A\|<1$ (Theorem 18.3 of \citep{rudin1987real}). Thus an invertible
matrix $\Gamma$ remains invertible when perturbed by an error $D$
with norm smaller than $1/\|\Gamma^{-1}\|$:
\begin{equation}
\|(\Gamma+D)^{-1}\|\leq\frac{\|\Gamma^{-1}\|}{1-\|\Gamma^{-1}\|\|D\|}.\label{eq:invert}
\end{equation}

\begin{proof}[Proof of proposition] Under Proposition \ref{prop:freqrecovery} we have
$\|\nu^{\lambda}-\hat{\nu}\|_{\infty}<2\epsilon(T)/T$, hence (\ref{eq:dhann})
implies that
\[
|\Gamma_{jk}-\hat{\Gamma}_{jk}|\leq\frac{1}{T}\left|\tilde{I}_{(0,T]}(\nu_{j}^{\lambda}-\nu_{k}^{\lambda})-\tilde{I}_{(0,T]}(\hat{\nu}_{j}-\hat{\nu}_{k})\right| < 4\pi\epsilon(T),
\]
so $\|\Gamma-\hat{\Gamma}\|\rightarrow0$ and result i) follows from
(\ref{eq:invert}). Henceforth we will assume that $\|\hat{\Gamma}^{-1}\|$
exists.

Next, observe that $y_{j}=\frac{1}{T}\int_{0}^{T}e^{-2\pi i\hat{\nu}_{j}t}dN(t)$
is the value of the periodogram (\ref{eq:periodogram}) at $\nu=\hat{\nu}_{j}$
when the rectangle window is used, so
\[
\begin{aligned}y_{j} & =\sum_{k=0}^{p}c_{k}^{\lambda}\frac{\tilde{I}_{(0,T]}(\hat{\nu}_{j}-\nu_{k}^{\lambda})}{T}+\frac{\tilde{\varepsilon}^{I_{(0,T]}}(\hat{\nu}_{j})}{T}\\
 & =\sum_{k=0}^{p}c_{k}^{\lambda}\frac{\tilde{I}_{(0,T]}(\hat{\nu}_{j}-\hat{\nu}_{k})}{T}+\frac{\tilde{\varepsilon}^{I_{(0,T]}}(\hat{\nu}_{j})}{T}\\
 & +\sum_{k=0}^{p}c_{k}^{\lambda}\frac{\tilde{I}_{(0,T]}(\hat{\nu}_{j}-\nu_{k}^{\lambda})-\tilde{I}_{(0,T]}(\hat{\nu}_{j}-\hat{\nu}_{k})}{T}\\
 & =(\hat{\Gamma}c^{\lambda})_{j}+(Ec^{\lambda})_{j}+\eta_{j},
\end{aligned}
\]
where $\eta$ is a vector whose $j$-th entry is $\tilde{\varepsilon}^{I_{(0,T]}}(\hat{\nu}_{j})/T$,
and $E$ is a matrix with $(j,k)$-entry $\frac{1}{T}\{\tilde{I}_{(0,T]}(\hat{\nu}_{j}-\nu_{k}^{\lambda})-\tilde{I}_{(0,T]}(\hat{\nu}_{j}-\hat{\nu}_{k})\}$.
It follows from (\ref{eq:dhann}) that $|E_{jk}|<2\pi\epsilon(T)$.
Furthermore, since the rectangle window satisfies the conditions in
Lemma \ref{lem:dualnorm},
\[
\begin{aligned}\|\hat{c}-c^{\lambda}\|_{\infty} & =\left\Vert \hat{\Gamma}^{-1}(\hat{\Gamma}c^{\lambda}+Ec^{\lambda}+\eta)-c^{\lambda}\right\Vert _{\infty}\\
 & \leq\|\hat{\Gamma}^{-1}\|\cdot\|Ec^{\lambda}+\eta\|_{\infty}\\
 & <\|\hat{\Gamma}^{-1}\|\left\{ 2\pi\|c^{\lambda}\|_{1}\epsilon(T)+4\alpha\bar{\Lambda}_{T}^{1/2}\left(\frac{\log T}{T}\right)^{1/2}\right\} .
\end{aligned}
\]
To complete the derivation of result ii), note that $\bar{\Lambda}_{T}\leq\frac{1}{T}\int_{0}^{T}|\lambda(u)|du\leq\|c^{\lambda}\|_{1}\leq\max(\|c^{\lambda}\|_{1},1)^{2}$
and $(\log T/T)^{1/2}<\epsilon(T)$.\end{proof}

\section{Asymptotic normality}\label{appendix:normality}

The derivation of the asymptotic normality results herein closely follows the setting and argument in \citet{shao2011}. We extend the result for the estimator obtained from the classic periodogram under the known $p$ setting to the windowed periodogram. Let $\hat{\nu}$ be the frequency estimates obtained from the windowed periodogram, and consider the cosine representation of the arrival rate in (\ref{eq:rate_spec}), $\lambda(t)=c_{0}^{\lambda}+\sum_{k=1}^{p/2}d_{k}^{\lambda}\cos(2\pi\nu_{k}^{\lambda}t+\phi_{k}^{\lambda})$, where we can assume without loss of generality that $\nu_{1}^{\lambda},\cdots,\nu_{p/2}^{\lambda}>0$.
\begin{prop}
	If $g(T)/T^{1/6}\to\infty$ as $T\rightarrow\infty$, then $T^{3/2}(\hat{\nu}-\nu^{\lambda})$
	is asymptotically normal with zero mean and covariance
	\[
	\begin{aligned} & \lim_{T\rightarrow\infty}\Cov\left[T^{3/2}(\hat{\nu}_{k}-\nu_{k}^{\lambda}),T^{3/2}(\hat{\nu}_{k'}-\nu_{k'}^{\lambda})\right]\\
	= & \frac{9}{1600d_{k}^{\lambda}d_{k'}^{\lambda}}\bigg((4\pi^{2}-30)\cos(\phi_{k}^{\lambda}-\phi_{k'}^{\lambda})c_{0}^{\lambda}+\sum_{j=1}^{p/2}d_{j}^{\lambda}\big((15-2\pi^{2})\cos(\phi_{j}^{\lambda}-\phi_{k}^{\lambda}-\phi_{k'}^{\lambda})\delta_{j,k+k'}\\
	& +\{(8\pi^{2}-15)\cos(\phi_{j}^{\lambda}-\phi_{k}^{\lambda}+\phi_{k'}^{\lambda})-6\pi^{2}\cos(\phi_{j}^{\lambda}+\phi_{k}^{\lambda}-\phi_{k'}^{\lambda})\}\delta_{j,k-k}\\
	& +\{(8\pi^{2}-15)\cos(\phi_{j}^{\lambda}+\phi_{k}^{\lambda}-\phi_{k'}^{\lambda})-6\pi^{2}\cos(\phi_{j}^{\lambda}-\phi_{k}^{\lambda}+\phi_{k'}^{\lambda})\}\delta_{j,k'-k}\big)\bigg),
	\end{aligned}
	\]
	where
	\begin{align*}
	\delta_{k,k'} & =I(\nu_{k}^{\lambda}=\nu_{k'}^{\lambda}),\\
	\delta_{j,k+k'} & =I(\nu_{j}^{\lambda}=\nu_{k}^{\lambda}+\nu_{k'}^{\lambda}),\\
	\delta_{j,k-k'} & =I(\nu_{j}^{\lambda}=\nu_{k}^{\lambda}-\nu_{k'}^{\lambda}),\\
	\delta_{j,k'-k} & =I(\nu_{j}^{\lambda}=\nu_{k'}^{\lambda}-\nu_{k}^{\lambda}).
	\end{align*}
\end{prop}

\begin{proof}
	The following quantities are asymptotically normal with mean zero:
	\begin{align*}
	U & =T^{-1/2}\int_{0}^{T}w(t)d\epsilon(t),\\
	V_{k} & =T^{-1/2}\int_{0}^{T}\cos(2\pi\nu_{k}^{\lambda}t)w(t)d\epsilon(t),\\
	W_{k} & =T^{-1/2}\int_{0}^{T}\sin(2\pi\nu_{k}^{\lambda}t)w(t)d\epsilon(t),\\
	X_{k} & =T^{-3/2}\int_{0}^{T}t\cos(2\pi\nu_{k}^{\lambda}t)w(t)d\epsilon(t),\\
	Y_{k} & =T^{-3/2}\int_{0}^{T}t\sin(2\pi\nu_{k}^{\lambda}t)w(t)d\epsilon(t),
	\end{align*}
	and the asymptotic covariance of $(U,V_{k},W_{k},X_{k},Y_{k})$ and
	$(U,V_{k'},W_{k'},X_{k'},Y_{k'})$ is
	\[
	(\Cov[U,V_{k}],\Cov[U,W_{k}])\to\frac{3d_{k}^{\lambda}}{16}\left\{ \cos(\phi_{k}^{\lambda}),-\sin(\phi_{k}^{\lambda})\right\} 
	\]
	\[
	(\Cov[U,X_{k}],\Cov[U,Y_{k}])\to\frac{3d_{k}^{\lambda}}{32}\left\{ \cos(\phi_{k}^{\lambda}),-\sin(\phi_{k}^{\lambda})\right\} 
	\]
	\[
	\Cov[V_{k},V_{k'}]\to\sum_{j=1}^{p/2}\frac{3d_{j}^{\lambda}}{32}\cos(\phi_{j}^{\lambda})(\delta_{j,k+k'}+\delta_{j,k-k'}+\delta_{j,k'-k})+\frac{3c_{0}^{\lambda}}{16}\delta_{k,k'}\triangleq E_{1}
	\]
	\[
	\Cov[V_{k},W_{k'}]\to\sum_{j=1}^{p/2}\frac{3d_{j}^{\lambda}}{32}\sin(\phi_{j}^{\lambda})(-\delta_{j,k+k'}+\delta_{j,k-k'}-\delta_{j,k'-k})\triangleq E_{2}
	\]
	\[
	(\Cov[V_{k},X_{k'}],\Cov[V_{k},Y_{k'}])\to\frac{1}{2}(E_{1},E_{2})
	\]
	\[
	\Cov[W_{k},W_{k'}]\to\sum_{j=1}^{p/2}\frac{3d_{j}^{\lambda}}{32}\cos(\phi_{j}^{\lambda})(-\delta_{j,k+k'}+\delta_{j,k-k'}+\delta_{j,k'-k})+\frac{3c_{0}^{\lambda}}{16}\delta_{k,k'}\triangleq E_{3}
	\]
	\[
	(\Cov[W_{k},X_{k'}],\Cov[W_{k},Y_{k'}])\to\frac{1}{2}(E_{2},E_{3})
	\]
	\[
	(\Cov[X_{k},X_{k'}],\Cov[X_{k},Y_{k'}],\Cov[Y_{k},Y_{k'}])\to\left(\frac{1}{3}-\frac{5}{8\pi^{2}}\right)(E_{1},E_{2},E_{3}).
	\]
	A Taylor expansion shows that
	\[
	\hat{\nu}_{k}-\nu_{k}^{\lambda}=-\frac{(|H(\nu_{k}^{\lambda})|^{2})'}{(|H(\bar{\nu}_{k}^{\lambda})|^{2})''}
	\]
	where $\bar{\nu}_{k}^{\lambda}$ is between $\nu_{k}^{\lambda}$ and
	$\hat{\nu}_{k}$. The nominator above is 
	\begin{align*}
	(|H(\nu_{k}^{\lambda})|^{2})' & =\frac{2}{T^{2}}\left(\int_{0}^{T}\sin(2\pi\nu_{k}^{\lambda}t)w(t)\lambda(t)dt+\int_{0}^{T}\sin(2\pi\nu_{k}^{\lambda}t)w(t)d\epsilon(t)\right)\\
	& \quad\times\left(\int_{0}^{T}2\pi t\cos(2\pi\nu_{k}^{\lambda}t)w(t)\lambda(t)dt+\int_{0}^{T}2\pi t\cos(2\pi\nu_{k}^{\lambda}t)w(t)d\epsilon(t)\right)\\
	& \quad-\frac{2}{T^{2}}\left(\int_{0}^{T}\cos(2\pi\nu_{k}^{\lambda}t)w(t)\lambda(t)dt+\int_{0}^{T}\cos(2\pi\nu_{k}^{\lambda}t)w(t)d\epsilon(t)\right)\\
	& \quad\times\left(\int_{0}^{T}2\pi t\sin(2\pi\nu_{k}^{\lambda}t)w(t)\lambda(t)dt+\int_{0}^{T}2\pi t\sin(2\pi\nu_{k}^{\lambda}t)w(t)d\epsilon(t)\right)\\
	& =\pi d_{k}^{\lambda}T^{1/2}\left(\frac{1}{2}W_{k}\cos(\phi_{k}^{\lambda})+\frac{1}{2}V_{k}\sin(\phi_{k}^{\lambda})-X_{k}\sin(\phi_{k}^{\lambda})-Y_{k}\cos(\phi_{k}^{\lambda})\right)\\
	& \quad+\frac{2}{T^{2}}\sum_{k=1}^{p/2}\sum_{l=1}^{p/2}d_{k}^{\lambda}d_{l}^{\lambda}\tilde{w}(\nu-\nu_{k}^{\lambda})\bar{\tilde{w}}(\nu-\nu_{l}^{\lambda})+4\pi\left(W_{k}X_{k}+V_{k}Y_{k}\right).
	\end{align*}
	By our hypothesis that $g(T)/T^{1/6}\rightarrow\infty$, the last
	two terms are $o(T^{1/2})$, and also $\bar{\nu}_{k}-\nu_{k}^{\lambda}=o(1/T)$.
	Hence
	\[
	\frac{1}{T}\sum_{l=1}^{p/2}d_{l}^{\lambda}\tilde{w}(\bar{\nu}_{k}-\nu_{l}^{\lambda})\to\frac{1}{2}d_{k}^{\lambda},
	\]
	\[
	\frac{1}{T^{2}}\sum_{l=1}^{p/2}d_{l}^{\lambda}\tilde{w}'(\bar{\nu}_{k}-\nu_{l}^{\lambda})\to-\frac{i\pi}{2}d_{k}^{\lambda},
	\]
	\[
	\frac{1}{T^{3}}\sum_{l=1}^{p/2}d_{l}^{\lambda}\tilde{w}''(\bar{\nu}_{k}-\nu_{l}^{\lambda})\to\left(-\frac{\pi^{2}}{2}+\frac{5}{6}\right)d_{k}^{\lambda},
	\]
	\[
	\frac{1}{T}\tilde{\epsilon}(\bar{\nu}_{k})=o(1),\ \frac{1}{T}\tilde{\epsilon}'(\bar{\nu}_{k})=o(T),\ \frac{1}{T}\tilde{\epsilon}''(\bar{\nu}_{k})=o(T^{2}),
	\]
	and therefore
	\begin{align*}
	\frac{(|H(\nu)|^{2})''}{T^{2}} & =\frac{H''(\nu)\bar{H}(\nu)+2H(\nu)\bar{H}(\nu)+\bar{H}''(\nu)H(\nu)}{T^{2}}\to\frac{5}{6}(d_{k}^{\lambda})^{2}.
	\end{align*}
	Putting everything together establishes the claimed asymptotic normality
	for
	\[
	T^{3/2}(\hat{\nu}_{k}-\nu_{k}^{\lambda})=-\frac{6\pi}{5d_{k}^{\lambda}}\left\{ \frac{1}{2}V_{k}\sin(\phi_{k}^{\lambda})+\frac{1}{2}W_{k}\cos(\phi_{k}^{\lambda})-X_{k}\sin(\phi_{k}^{\lambda})-Y_{k}\cos(\phi_{k}^{\lambda})\right\} .
	\]
\end{proof}
If $g(T)\rightarrow\infty$ as $T\rightarrow\infty$, the asymptotic behaviour of the coefficient
estimate $\hat{c}=\hat{\Gamma}^{-1}y$ given by (\ref{eq:c-hat}) is
identical to that of $y$ since $\hat{\Gamma}$ converges to an orthonormal
design. A standard application of the delta method then establishes the asymptotic normality for the real and imaginary parts of
\[
\begin{aligned}y_{k} & =\frac{1}{T}\int_{0}^{T}e^{-2\pi i\hat{\nu}_{k}t}dN(t)\\
&  =\frac{1}{T}\left(\int_{0}^{T}\cos(2\pi\hat{\nu}_{k}t)\lambda(t)dt+\int_{0}^{T}\cos(2\pi\hat{\nu}_{k}t)d\varepsilon(t)\right)\\
& -i\frac{1}{T}\left(\int_{0}^{T}\sin(2\pi\hat{\nu}_{k}t)\lambda(t)dt+\int_{0}^{T}\sin(2\pi\hat{\nu}_{k}t)d\varepsilon(t)\right).
\end{aligned}
\]
\begin{prop}
	If $g(T)/T^{1/6}\to\infty$ as $T\rightarrow\infty$, then $T^{1/2}\left(\mathrm{Re}(y)-(1/T)\int_{0}^{T}\cos(2\pi\nu_{k}^{\lambda}t)\lambda(t)dt\right)$
	and $T^{1/2}\left(\mathrm{Im}(y)+(1/T)\int_{0}^{T}\sin(2\pi\nu_{k}^{\lambda}t)\lambda(t)dt\right)$
	are asymptotically normal with zero mean and covariance given by 
	\[
	\begin{aligned} & \lim_{T\rightarrow\infty}\Cov\left[T^{-1/2}\int_{0}^{T}\cos(2\pi\hat{\nu}_{k}t)\lambda(t)dt,T^{-1/2}\int_{0}^{T}\cos(2\pi\hat{\nu}_{k'}t)\lambda(t)dt\right]\\
	= & \frac{\pi^{2}}{4}\sin(\phi_{k}^{\lambda})\sin(\phi_{k'}^{\lambda})\lim_{T\rightarrow\infty}\Cov\left[T^{3/2}\hat{\nu}_{k},T^{3/2}\hat{\nu}_{k'}\right]
	\end{aligned}
	\]
	\[
	\begin{aligned} & \lim_{T\rightarrow\infty}\Cov\left[T^{-1/2}\int_{0}^{T}\cos(2\pi\hat{\nu}_{k}t)d\varepsilon(t),T^{-1/2}\int_{0}^{T}\cos(2\pi\hat{\nu}_{k'}t)d\varepsilon(t)\right]\\
	= & \sum_{j=1}^{p/2}\frac{d_{j}^{\lambda}}{4}\cos(\phi_{j}^{\lambda})(\delta_{j,k+k'}+\delta_{j,k-k'}+\delta_{j,k'-k})+\frac{c_{0}^{\lambda}}{2}\delta_{k,k'}
	\end{aligned}
	\]
	\[
	\begin{aligned} & \lim_{T\rightarrow\infty}\Cov\left[T^{-1/2}\int_{0}^{T}\cos(2\pi\hat{\nu}_{k}t)\lambda(t)dt,T^{-1/2}\int_{0}^{T}\cos(2\pi\hat{\nu}_{k'}t)d\varepsilon(t)\right]\\
	= & \frac{\pi}{2}\sin(\phi^{\lambda_{k}})\lim_{T\rightarrow\infty}\Cov\left[T^{3/2}\hat{\nu}_{k},T^{-1/2}\int_{0}^{T}\cos(2\pi\nu_{k'}^{\lambda}t)d\varepsilon(t)\right]
	\end{aligned}
	\]
	\[
	\begin{aligned} & \lim_{T\rightarrow\infty}\Cov\left[T^{-1/2}\int_{0}^{T}\cos(2\pi\hat{\nu}_{k})\lambda(t)dt,-T^{-1/2}\int_{0}^{T}\sin(2\pi\hat{\nu}_{k'})\lambda(t)dt\right]\\
	= & \frac{\pi^{2}}{4}\sin(\phi_{k}^{\lambda})\cos(\phi_{k'}^{\lambda})\lim_{T\rightarrow\infty}\Cov\left[T^{3/2}\hat{\nu}_{k},T^{3/2}\hat{\nu}_{k'}\right]
	\end{aligned}
	\]
	\[
	\begin{aligned} & \lim_{T\rightarrow\infty}\Cov\left[T^{-1/2}\int_{0}^{T}\cos(2\pi\hat{\nu}_{k})d\varepsilon(t),-T^{-1/2}\int_{0}^{T}\sin(2\pi\hat{\nu}_{k'})\lambda(t)dt\right]\\
	= & \frac{\pi}{2}\cos(\phi_{k'}^{\lambda})\lim_{T\rightarrow\infty}\Cov\left[T^{-1/2}\int_{0}^{T}\cos(2\pi\nu_{k}^{\lambda}t)d\varepsilon(t),T^{3/2}\hat{\nu}_{k'}\right]
	\end{aligned}
	\]
	\[
	\begin{aligned} & \lim_{T\rightarrow\infty}\Cov\left[T^{-1/2}\int_{0}^{T}\cos(2\pi\hat{\nu}_{k}t)\lambda(t)dt,-T^{-1/2}\int_{0}^{T}\sin(2\pi\hat{\nu}_{k'}t)d\varepsilon(t)\right]\\
	= & \frac{\pi}{2}\sin(\phi^{\lambda_{k}})\lim_{T\rightarrow\infty}\Cov\left[T^{3/2}\hat{\nu}_{k},T^{-1/2}\int_{0}^{T}\sin(2\pi\nu_{k'}^{\lambda}t)d\varepsilon(t)\right]
	\end{aligned}
	\]
	\[
	\begin{aligned} & \lim_{T\rightarrow\infty}\Cov\left[T^{-1/2}\int_{0}^{T}\cos(2\pi\hat{\nu}_{k}t)d\varepsilon(t),-T^{-1/2}\int_{0}^{T}\sin(2\pi\hat{\nu}_{k'}t)d\varepsilon(t)\right]\\
	= & \sum_{j=1}^{p/2}\frac{d_{j}^{\lambda}}{4}\sin(\phi_{j}^{\lambda})(\delta_{j,k+k'}-\delta_{j,k-k'}+\delta_{j,k'-k})
	\end{aligned}
	\]
	\[
	\begin{aligned} & \lim_{T\rightarrow\infty}\Cov\left[-T^{-1/2}\int_{0}^{T}\sin(2\pi\hat{\nu}_{k}t)d\varepsilon(t),-T^{-1/2}\int_{0}^{T}\sin(2\pi\hat{\nu}_{k'}t)d\varepsilon(t)\right]\\
	= & \sum_{j=1}^{p/2}\frac{d_{j}^{\lambda}}{4}\cos(\phi_{j}^{\lambda})(-\delta_{j,k+k'}+\delta_{j,k-k'}+\delta_{j,k'-k})+\frac{c_{0}^{\lambda}}{2}\delta_{k,k'}
	\end{aligned}
	\]
	\[
	\begin{aligned} & \lim_{T\rightarrow\infty}\Cov\left[-T^{-1/2}\int_{0}^{T}\sin(2\pi\hat{\nu}_{k})\lambda(t)dt,-T^{-1/2}\int_{0}^{T}\sin(2\pi\hat{\nu}_{k'}t)d\varepsilon(t)\right]\\
	= & \frac{\pi}{2}\cos(\phi_{k}^{\lambda})\lim_{T\rightarrow\infty}\Cov\left[T^{3/2}\hat{\nu}_{k},-T^{-1/2}\int_{0}^{T}\sin(2\pi\nu_{k'}t)d\varepsilon(t)\right]
	\end{aligned}
	\]
	\[
	\begin{aligned} & \lim_{T\rightarrow\infty}\Cov\left[-T^{-1/2}\int_{0}^{T}\sin(2\pi\hat{\nu}_{k})\lambda(t)dt,-T^{-1/2}\int_{0}^{T}\sin(2\pi\hat{\nu}_{k'})\lambda(t)dt\right]\\
	= & \frac{\pi^{2}}{4}\cos(\phi_{k}^{\lambda})\cos(\phi_{k'}^{\lambda})\lim_{T\rightarrow\infty}\Cov\left[T^{3/2}\hat{\nu}_{k},T^{3/2}\hat{\nu}_{k'}\right].
	\end{aligned}
	\]
\end{prop}

\bibliography{ref}

\begin{thebibliography}{22}
\providecommand{\natexlab}[1]{#1}
\providecommand{\url}[1]{\texttt{#1}}
\expandafter\ifx\csname urlstyle\endcsname\relax
  \providecommand{\doi}[1]{doi: #1}\else
  \providecommand{\doi}{doi: \begingroup \urlstyle{rm}\Url}\fi

\bibitem[Bartlett(1963)]{bartlett}
MS~Bartlett.
\newblock The spectral analysis of point processes.
\newblock \emph{J. R. Statist. Soc. B}, 25\penalty0 (2):\penalty0 264--296,
  1963.

\bibitem[Bebbington and Zitikis(2004)]{bebb}
M~Bebbington and R~Zitikis.
\newblock A robust heuristic estimator for the period of a {Poisson} intensity
  function.
\newblock \emph{Methodol. Comput. Appl. Probab.}, 6\penalty0 (4):\penalty0
  441--462, 2004.

\bibitem[Belitser et~al.(2013)Belitser, Serra, and van Zanten]{belitser}
E~Belitser, P~Serra, and H~van Zanten.
\newblock Estimating the period of a cyclic non-homogeneous {Poisson} process.
\newblock \emph{Scandinavian Journal of Statistics}, 40\penalty0 (2):\penalty0
  204--218, 2013.

\bibitem[Bhaskar et~al.(2013)Bhaskar, Tang, and Recht]{bhaskar}
BN~Bhaskar, G~Tang, and B~Recht.
\newblock Atomic norm denoising with applications to line spectral estimation.
\newblock \emph{IEEE Trans. Sig. Process.}, 61\penalty0 (23):\penalty0
  5987--5999, 2013.

\bibitem[Brown et~al.(2010)Brown, Cai, Zhang, Zhao, and Zhou]{brown2010}
L~Brown, T~Cai, R~Zhang, L~Zhao, and H~Zhou.
\newblock The root--unroot algorithm for density estimation as implemented via
  wavelet block thresholding.
\newblock \emph{Probability theory and related fields}, 146\penalty0
  (3-4):\penalty0 401--433, 2010.

\bibitem[Cand{\`e}s and Fernandez-Granda(2013)]{candes}
EJ~Cand{\`e}s and C~Fernandez-Granda.
\newblock Super-resolution from noisy data.
\newblock \emph{Journal of Fourier Analysis and Applications}, 19\penalty0
  (6):\penalty0 1229--1254, 2013.

\bibitem[Chen et~al.(2018)Chen, Lee, and Shen]{CLS2018}
N~Chen, DKK Lee, and HP~Shen.
\newblock Can customer arrival rates be modelled by sine waves?
\newblock \emph{Working paper}, 2018.

\bibitem[Donoho and Johnstone(1994)]{dj94}
DL~Donoho and JM~Johnstone.
\newblock Ideal spatial adaptation by wavelet shrinkage.
\newblock \emph{Biometrika}, 81\penalty0 (3):\penalty0 425--455, 1994.

\bibitem[Dutt and Rokhlin(1993)]{nufft}
A~Dutt and V~Rokhlin.
\newblock Fast {F}ourier transforms for nonequispaced data.
\newblock \emph{SIAM J. Sci. Comput.}, 14\penalty0 (6):\penalty0 1368--1393,
  1993.

\bibitem[Fernandez-Granda(2013)]{fernandez2013}
C~Fernandez-Granda.
\newblock Support detection in super-resolution.
\newblock In \emph{Proceedings of the 10th International Conference on Sampling
  Theory and Applications}, pages 145--148, 2013.

\bibitem[Helmers and Mangku(2003)]{helmers}
R~Helmers and IW~Mangku.
\newblock On estimating the period of a cyclic {Poisson} process.
\newblock \emph{Mathematical Statistics and Applications: Festschrift for
  Constance van Eeden (eds Moore, Froda, Leger), IMS Beachwood}, pages
  345--356, 2003.

\bibitem[Lewis(1970)]{lewis}
PAW Lewis.
\newblock Remarks on the theory, computation and application of the spectral
  analysis of series of events.
\newblock \emph{Journal of Sound and Vibration}, 12\penalty0 (3):\penalty0
  353--375, 1970.

\bibitem[Li(2014)]{li2014}
TH~Li.
\newblock \emph{Time series with mixed spectra}.
\newblock CRC Press, 2014.

\bibitem[Moitra(2015)]{moitra2015}
A~Moitra.
\newblock Super-resolution, extremal functions and the condition number of
  vandermonde matrices.
\newblock In \emph{Proceedings of the Forty-Seventh Annual ACM on Symposium on
  Theory of Computing}, pages 821--830. ACM, 2015.

\bibitem[Osipov et~al.(2013)Osipov, Rokhlin, and Xiao]{osipov}
A~Osipov, V~Rokhlin, and H~Xiao.
\newblock Prolate spheroidal wave functions of order zero.
\newblock \emph{Springer Ser. Appl. Math. Sci}, 187, 2013.

\bibitem[Prabhu(2013)]{prabhu}
K~Prabhu.
\newblock \emph{Window functions and their applications in signal processing}.
\newblock CRC Press, 2013.

\bibitem[Rice and Rosenblatt(1988)]{rice}
JA~Rice and M~Rosenblatt.
\newblock On frequency estimation.
\newblock \emph{Biometrika}, 75\penalty0 (3):\penalty0 477--484, 1988.

\bibitem[Rudin(1987)]{rudin1987real}
W~Rudin.
\newblock \emph{Real and complex analysis}.
\newblock McGraw-Hill, 1987.

\bibitem[Shao(2010)]{shao2010}
N~Shao.
\newblock \emph{Modeling Almost Periodicity in Point Processes}.
\newblock PhD thesis, University of California, Riverside, 2010.

\bibitem[Shao and Lii(2011)]{shao2011}
N~Shao and KS~Lii.
\newblock Modelling non-homogeneous {Poisson} processes with almost periodic
  intensity functions.
\newblock \emph{J. R. Statist. Soc. B}, 73\penalty0 (1):\penalty0 99--122,
  2011.

\bibitem[Tang et~al.(2015)Tang, Bhaskar, and Recht]{tang}
G~Tang, BN~Bhaskar, and B~Recht.
\newblock Near minimax line spectral estimation.
\newblock \emph{IEEE Trans. Inf. Theory}, 61\penalty0 (1):\penalty0 499--512,
  2015.

\bibitem[Vere-Jones(1982)]{vere-jones}
D~Vere-Jones.
\newblock On the estimation of frequency in point-process data.
\newblock \emph{J. of Appl. Probab.}, pages 383--394, 1982.

\end{thebibliography}
\bibliographystyle{plainnat}

\end{document}